\documentclass[runningheads]{llncs}

\usepackage{float}
\usepackage{eccv}

\def\E{\mathbb{E}}
\def\hf{\widehat{f}}
\def\hp{\widehat{p}}
\def\bs{\boldsymbol}
\def\bsalpha{\bs{\alpha}}
\def\bstau{\bs{\tau}}

\def\bss{\bs{s}}
\def\bsomega{\bs{\omega}}
\def\bsx{\bs{x}}
\def\tbsx{\tilde{\bsx}}

\newcommand\restr[2]{{
  \left.\kern-\nulldelimiterspace 
  #1 
  \littletaller 
  \right|_{#2} 
  }}

\newcommand{\littletaller}{\mathchoice{\vphantom{\big|}}{}{}{}}

\usepackage{eccvabbrv}

\usepackage{graphicx}
\usepackage{booktabs}

\usepackage[accsupp]{axessibility}  

\usepackage{hyperref}

\usepackage{multicol}
\usepackage{multirow}
\usepackage{orcidlink}

\begin{document}

\title{On Spectral Properties of Gradient-based Explanation Methods}

\titlerunning{Spectral Properties of Gradient-based Explanation Methods}

\author{Amir Mehrpanah\inst{1}\orcidlink{0000-0002-6193-7126} \and
Erik Englesson\inst{1}\orcidlink{0000-0003-4535-2520} \and
Hossein Azizpour\inst{1}\orcidlink{0000-0001-5211-6388}}

\authorrunning{A.~Mehrpanah et al.}

\institute{Department of Computer Science, Royal Institute of Technology, Stockholm, Sweden 
\email{\{amirme,engless,azizpour\}@kth.se}}

\maketitle

\begin{abstract}

Understanding the behavior of deep networks is crucial to increase our confidence in their results. Despite an extensive body of work for explaining their predictions, researchers have faced reliability issues, which can be attributed to insufficient formalism. In our research, we adopt novel probabilistic and spectral perspectives to formally analyze explanation methods. Our study reveals a pervasive spectral bias stemming from the use of gradient, and sheds light on some common design choices that have been discovered experimentally, in particular, the use of squared gradient and input perturbation. We further characterize how the choice of perturbation hyperparameters in explanation methods, such as SmoothGrad, can lead to inconsistent explanations and introduce two remedies based on our proposed formalism: (i) a mechanism to determine a standard perturbation scale, and (ii) an aggregation method which we call SpectralLens. Finally, we substantiate our theoretical results through quantitative evaluations.
    
\keywords{Probabilistic Machine Learning \and Gradient-based Explanation Methods \and Probabilistic Pixel Attribution Techniques \and Explainability \and Deep Neural Networks \and Spectral Analysis}
\end{abstract}

\section{Introduction}
\label{introduction}
Ensuring the interpretability of machine learning models is crucial for user's trust and understanding \cite{ras_explanation_2018,adebayo_sanity_2020,arrieta_explainable_2019}. 
While considerable efforts have been dedicated to advancing explainability in the field \cite{zeiler_visualizing_2013,ribeiro_why_2016,lundberg_unified_2017,sundararajan_axiomatic_2017,rudin_stop_2019,selvaraju_grad-cam_2020}, there is a notable gap in systematically formalizing and consolidating existing knowledge \cite{ancona_towards_2018,bansal_sam_2020,chen_true_2020,covert_explaining_2022}.
Insufficient formal analysis leaves too much room for different interpretations of the explanations, making them unreliable \cite{harel_inherent_2022} which is detrimental to the very use case of explanations, i.e., human-critical applications.
Therefore, more attention is essential to retrospectively analyze and structure the wealth of work done, creating a cohesive foundation for explainability \cite{marques-silva_delivering_2022,wilming_theoretical_2023}.
\\One main reason for the unreliability of explanation methods is their susceptibility to hyperparameters, leading to inconsistent explanations \cite{kindermans_reliability_2017,alvarez-melis_robustness_2018,brughmans_disagreement_2023,rong_consistent_2022,krishna_disagreement_2022}, see Fig.~\ref{fig:example-sgsq-row}. Among the several works that formally study explanation methods~\cite{sundararajan_axiomatic_2017,covert_explaining_2022,lundberg_unified_2017}, to our knowledge, none shed light on such inconsistencies. Therefore, in this work, we adopt a novel probabilistic and spectral perspective to characterize the observed inconsistencies and propose preliminary solutions. 

\subsubsection{Our Contribution.}
We iterate our contributions in three parts:
\begin{itemize}
    \item \textbf{Analysis:} We introduce a probabilistic representation for post-hoc gradient-based explanation methods in \cref{probabilistic-representation}, and perform a spectral analysis of them in \cref{spectral-representation-explanation}.  
    \item \textbf{Identified Issues:} Our examination reveals a shared spectral bias stemming from the fundamental assumption of utilizing gradients to measure feature contributions. We show that the combined spectral effects of gradient and perturbation lead to band-pass filters modulating the explanations, which could lead to inconsistencies in the explanations, discussed in \cref{spectral-properties-explainer-perturbation}.
    Moreover, we conduct a spectral analysis of a few specific perturbation distributions that have been proposed in the literature in \cref{spectral-analysis-specific-kernels}. 
    \item \textbf{Solutions and Evaluations:} To address inconsistencies, we investigate two straightforward solutions that are finding an optimal perturbation scale based on the cosine similarity of the perturbation and the classifier in the frequency domain in \cref{kc-similarity-subsection}, and an aggregation mechanism named SpectralLens in \cref{aggregation-of-information}.
    Finally, we evaluate the proposed methods using the pixel removal strategy \cite{rong_consistent_2022}, acknowledging their effectiveness and limitations under \cref{evaluation}.
\end{itemize}

\begin{figure*}[t]
    \centering
    \begin{subfigure}[b]{0.24\textwidth}
    \includegraphics[width=\textwidth]{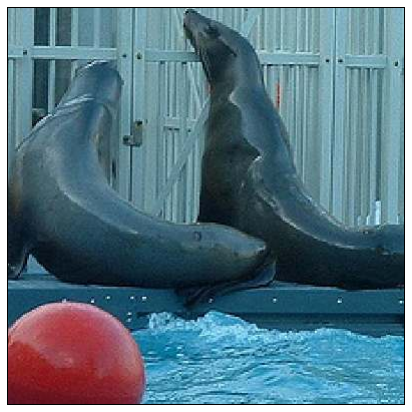}
    \caption{}
    \end{subfigure}
    \begin{subfigure}[b]{0.24\textwidth}
    \includegraphics[width=\textwidth]{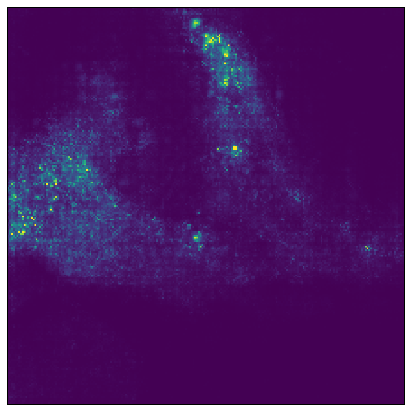}
    \caption{}
    \end{subfigure}
    \begin{subfigure}[b]{0.24\textwidth}
    \includegraphics[width=\textwidth]{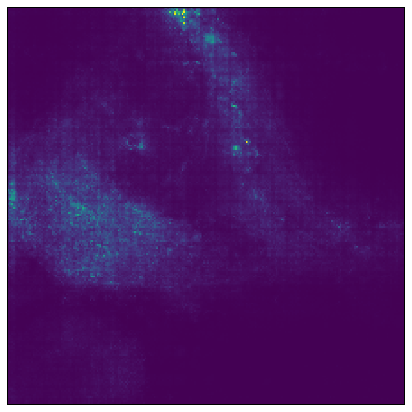}
    \caption{}
    \end{subfigure}
    \begin{subfigure}[b]{0.24\textwidth}
    \includegraphics[width=\textwidth]{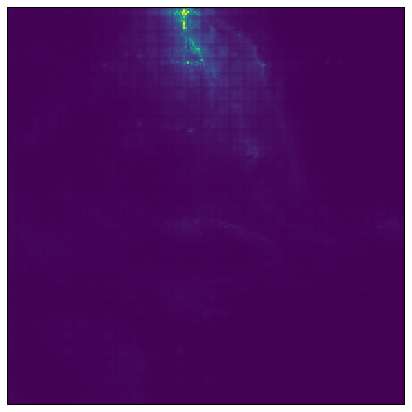}
    \caption{}
    \end{subfigure}
    \caption{The figure illustrates sample explanations generated by varying hyperparameters of SmoothGrad-Squared \cite{hooker_benchmark_2019,smilkov_smoothgrad_2017} for ResNet-50 on ImageNet (see Fig.~\ref{fig:example-sgsq-row-appendix} in the Appendix for more examples). It is evident from the figure that using different hyperparameters results in diverse explanations. Such variations have been demonstrated in recent studies for other explanation techniques as well \cite{kindermans_reliability_2017,alvarez-melis_robustness_2018,brughmans_disagreement_2023,rong_consistent_2022}.
    }
    \label{fig:example-sgsq-row}
\end{figure*}

\section{Background}
\subsubsection{Notation.}
We consider a classification setting, that is, modeling $p(y|\bsx)$ for labels $y\in\{1,\dots,k\}$ and $\bsx\in\mathbb{R}^n$ the set of input images. We denote the $\log p(y|\bsx)$  by $f$.
We denote standard Fourier bases by $\{\psi(\bsomega,\bsx)\}_{\bsomega\in\mathbb{R}^n}$, according to which, under some regularity conditions \cite{rahaman_spectral_2019}, 
$f$ can be decomposed as $f(\bsx) = \int \hf(\bsomega) \psi(\bsomega,\bsx) d\bsomega$.
In this representation, $\hf$ denotes the Fourier transform (FT) of $f$. We denote the power spectral density (PSD) of $f$ by $\operatorname{S}_{f}(\bsomega)=|\hf(\bsomega)|^2$.

\subsection{A Probabilistic Representation of Explanations}
\label{probabilistic-representation}
In this view, each explanation method consists of three components: (1) a distribution for perturbation of the input, (2) an explainer, and (3) a summary statistic.
We use $p(\tbsx|\bsx,\sigma)$ to denote the distribution of perturbations of the input sample parameterized by $\sigma$, referred to as the perturbation scale.
A definition of ``explanation'' is recovering the contribution of features to a prediction, which emanates from human's natural counterfactual reasoning \cite{covert_explaining_2022}. Therefore, the explainer is usually the gradient with respect to input $\nabla_{\bsx}$, or approximations of it, \eg the finite difference, given by $\Delta_{\bsx} f(\boldsymbol{x})= \frac{f(\tbsx)-f(\bsx)}{\tbsx-\bsx}$ (see \cref{appendix:finite_difference}).
Finally, we assume that the summary statistic is the expectation of the outputs, or the squared outputs, of the explainer, \eg $\left(\nabla f(\tbsx)\right)^2$, over the chosen perturbation distribution, $p(\tbsx|\bsx,\sigma)$, denoted as $\E_{p(\tbsx)}\left[(\nabla f(\tbsx))^2\right]$, where the operations are carried out pixelwise. This is referred to as the explanation of $f$ at $\bsx$ with a hyperparameter $\sigma$.
This probabilistic view can, in fact, unify different explanation methods into a single form (see \cref{appendix:unifying}). We provide theoretical justifications for why using squared gradients is preferable to gradients, explaining the empirical findings \cite{hooker_benchmark_2019}, and will discuss both cases in \cref{spectral-representation-explanation}.

\section{Spectral Analysis of Explanations}
\label{spectral_bias}
This section demonstrates the spectral property in explanations when using the gradient explainer and its interaction with perturbations, where a significant difference in the frequency of learned features leads to a selective attribution of features.
In the following sections, we draw on insights from spectral analysis of signals and systems theory, utilizing this viewpoint to gain a deeper understanding of explanation methods. We demonstrate a spectral bias related to gradient-based explanation approaches, identify its source, and interpret various methods in the literature as efforts to alleviate it through the incorporation of low-pass filters, either explicitly \cite{smilkov_smoothgrad_2017,sundararajan_axiomatic_2017} or implicitly \cite{zeiler_visualizing_2013,petsiuk_rise_2018}. 
We illustrate that the inconsistencies in explanation techniques stem from different strategies for mitigating the spectral bias. Finally, we suggest two straightforward approaches to tackle these inconsistencies: determining an optimal perturbation scale in \cref{kc-similarity-subsection} and aggregating the explanations in \cref{aggregation-of-information}.

Writing a spectral representation for prediction, as shown in \cref{spectral-representation-prediction}, the model output with a perturbation distribution can be written as a convolution between the function and the perturbation distribution $f*p$, therefore we may use ``perturbation kernel'' and ``perturbation distribution'' interchangeably.

\subsection{Spectral Representation of Explainers}
\label{spectral-representation-explanation}
Here, by drawing a connection between explainability techniques and the literature on spectral signal processing \cite{blackman_measurement_1958,marmarelis_white-noise_1978,stoica_spectral_2005}, we take a novel avenue for our analysis.
This connection enables us to have a spectral representation of the explanations, as discussed in the next sections (see \cref{appendix:derivations} for the derivations).
It is important to note that the Fourier bases in our work is 1D and per pixel, and should not be confused with the traditional Fourier bases used in image processing literature that expands along spatial dimensions.

\subsubsection{Gradients in Spectral Domain.}
For explanation methods that use gradients and a perturbation kernel, $p(\tbsx)$ we have the following (see \cref{appendix:derivations-gradient}).
\begin{align}
    \E_{p(\tbsx)}[\nabla f(\tbsx)] &=\int i2\pi\bsomega \hf(\bsomega) \hp(\bsomega) d \bsomega\\
    &=-2\pi\left[\int \bsomega(\hf_E(\bsomega)\hp_O(\bsomega) + \hf_O(\bsomega)\hp_E(\bsomega)) d \bsomega\right]
    \label{eq:general-spectral-representation-gradients}
\end{align}
where $\widehat{p}$ denotes the FT of the perturbation kernel. We used $\hf_E$ and $\hf_O$ to denote a function $\hf$ is decomposed into real even and imaginary odd components $\hf=i\hf_O+\hf_E$\footnote{Also note that as we are using real valued functions, other components \ie real odd and imaginary even, are zero.}. The negative sign suggests that in case of using the gradients, we need to negate the actual values to obtain rankings. We will elaborate more on the negation of the attributions in \cref{evaluation}.

\begin{remark}
    Note the symmetry in the form of the spectral representation of the gradients. This form implies that it is inevitable for any perturbation kernel to extract asymmetric information about even and odd parts of the classifier $f$. Therefore, the gradient explainer with a kernel symmetric about the sample being explained, only extracts information about the odd parts of the classifier.
    This is problematic because the attribution is expected to reflect the total contribution to the prediction.
\end{remark}
We will use this interesting symmetry later in \cref{rect-kernel-analysis} for analysis of a specific perturbation kernel.

\subsubsection{Squared Gradients in Spectral Domain.}
On the other hand, for explanation methods which use squared gradients, we have the following (see \cref{appendix:derivations-sq-gradient}).
\begin{equation}
    \E_{p(\tbsx)}\left[(\nabla f(\tbsx))^2\right]=4\pi^2\int\|\bsomega\|^2 \operatorname{S}_f(\bsomega) \operatorname{S}_{\sqrt{p}}(\bsomega) d\bsomega
    \label{eq:general-spectral-representation-sqgradients}
\end{equation}
where $\operatorname{S}_{f}$ and $\operatorname{S}_{\sqrt{p}}$ denote the PSD of the model and square root of the perturbation kernel respectively.

We discuss the benefits of using squared gradients as in \cref{eq:general-spectral-representation-sqgradients} over gradients as in \cref{eq:general-spectral-representation-gradients} later in \cref{spectral-analysis-specific-kernels}. However, it can be noticed here that the squared gradients, unlike gradients, are always non-negative, extract information about the PSD of $f$, and do not have a complex symmetric structure.

\subsection{Gradient Explainers with Perturbations Are Band-Pass Filters}
\label{spectral-properties-explainer-perturbation}
In our probabilistic framework, various explanation methods can be decomposed into three components (1) an explainer, (2) a perturbation kernel and (3) a summary statistic. In this section, we analyze the spectral effects of the two core components, namely gradient and perturbation.
Importantly, we show that gradient acts as a high-pass filter, and perturbation acts as a low-pass filter.
Before that, we need a more formal definition for a high-pass and low-pass filter that suits our study. Assuming that $\widehat{m}(\bsomega)$ denotes the FT of a perturbation kernel, we define low- and high-pass as follows.
\begin{definition}
    A filter $\widehat{m}(\bsomega)$ is low pass, if for a given $M>0$ there exists a threshold $\bsomega^*$ such that for any $\|\bsomega\|>\|\bsomega^*\|$, we have $\|\widehat{m}(\bsomega)\|<M$, and it is high pass if such a condition holds for $\|\bsomega\|<\|\bsomega^*\|$.
\end{definition}

For better readability, the proofs for the theorems and propositions are outlined in \cref{prop:mitigation-proof}, and we provide intuition as we progress. 

\begin{theorem}
    \label{prop:mitigation}
    Any perturbation kernel $p(\tbsx)$ mitigates the attribution of high-frequency features; simply put, \textbf{any perturbation is a low-pass filter}.
\end{theorem}

The intuition is that any modification to a signal, would only remove information, and, as the intensity of perturbation increases, the removal goes from high-frequency components toward lower frequencies.

\begin{theorem}
    \label{prop:amplification}
    The gradient operator amplifies the attribution of high-frequency features; simply put, \textbf{the gradient operator is a high-pass filter}.
\end{theorem}

From the stated theorems, one can see that a gradient-based explanation with perturbation creates a band-pass filter, whose shape is controlled by the perturbation hyperparameters (see \cref{fig:bands,fig:concept}). 

\begin{remark}
    This amplification-mitigation trade-off can be seen in several explanation methods, where perturbation kernels (low-pass filters) are utilized to mitigate the amplification effect (the high-pass filter) caused by the gradient. Therefore, the combined effect of the two filters usually forms \textbf{a band-pass filter} (see \cref{appendix:unifying} for a non-comprehensive list).
\end{remark}

\subsection{Spectral Analysis of Specific Perturbation Kernels}
\label{spectral-analysis-specific-kernels}
In this section, we analyze the properties of special kernels used in the literature with remarks or propositions that justify, and thus shed light on, previous work's experimental observations. In most cases, we expand upon the squared gradients, \cref{eq:general-spectral-representation-sqgradients}, due to its analytical simplicity.

\subsubsection{VanillaGrad via a Dirac Kernel.}
\label{sect:vanilla-grad}
The very naive explanation method can be obtained with a Dirac kernel $p(\tbsx)=\delta_{\bsx}(\tbsx)$ known as \textbf{VanillaGrad} (VG). Analyzing the squared version of VG and following \cref{eq:general-spectral-representation-sqgradients}, we know that the FT of the Dirac kernel is $\operatorname{S}_{\sqrt{p}}(\bsomega)=\mathbf{1}(\bsomega)$, where $\mathbf{1}(\bsomega)$ is the constant function mapping all values to 1. Thus, it can be seen that we recover the scaled PSD of $f$, \ie $\|\bsomega\|^2\operatorname{S}_f(\bsomega)$.
\begin{remark}
Scaling $\hf$ by frequency can be interpreted as the bias of the gradient explainer towards exaggeration of the attribution of high-frequency features \cite{hou_saliency_2007}.
Beside other justifications for gradient sparsity \cite{balduzzi_shattered_2018}, spectral representation of VG$^2$ and domination of high-frequency features give an orthogonal \textit{justification for why VanillaGrad explanations being usually scattered and noisy} (see \cref{fig:example-vgsq-row}).    
\end{remark}

\subsubsection{SmoothGrad via a Gaussian Kernel.}
\label{sect:guassian-kernel}
Letting $p(\tbsx) = \mathcal{N}(\bsx,\sigma^2)$, \ie a Gaussian with variance $\sigma^2$, centered on the sample being explained, we recover variants of an explanation method known as SmoothGrad (SG) or the squared gradient version SmoothGrad-Squared (SG$^2$) \cite{smilkov_smoothgrad_2017,hooker_benchmark_2019}, we can rewrite \cref{eq:general-spectral-representation-sqgradients} as follows:
\begin{equation}
    \E_{p(\tbsx)}\left[(\nabla f(\tbsx))^2\right]\propto \int \|\bsomega\|^2\operatorname{S}_f(\bsomega) e^{-8\pi^2\sigma^2\|\bsomega\|^2} d\bsomega
    \label{eq:spectral-representation-gaussian}
\end{equation}
This shows that the contributions $\operatorname{S}_f(\bsomega)$ are modulated by the combined spectral effect of the Gaussian perturbation kernel and gradient operator, which is a band pass filter of the form $\|\bsomega\|^2e^{-8\pi^2\sigma^2\|\bsomega\|^2}$, whose shape is controlled by the kernel variance (see \cref{fig:bands}).
The explanations could be inconsistent, but generally become smoother as $\sigma$ increases, which is a result of the shrinkage of the attribution of high-frequency features that can dominate the explanations with a small $\sigma$ (see \cref{fig:example-sgsq-row}).
\begin{figure}
    \centering
    \includegraphics[width=0.45\textwidth]{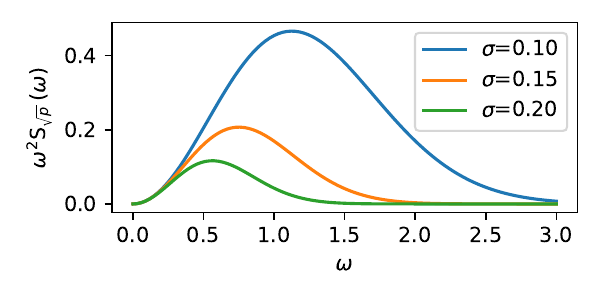}
    \caption{A visualization of frequency bands, where $\operatorname{S}_{\sqrt{p}}(\omega)$ is defined according to \cref{eq:spectral-representation-gaussian}, show that change in the noise scale $\sigma$ leads to band-pass filters with different modes, hence, representing information of different frequencies.}
    \label{fig:bands}
\end{figure}
\begin{figure}
    \centering
    \begin{subfigure}[b]{0.28\textwidth}    
    \includegraphics[width=\textwidth]{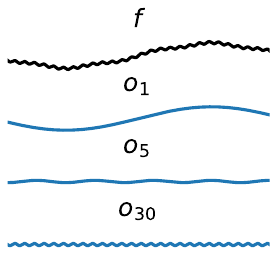}
    \caption{}
    \end{subfigure}
    \begin{subfigure}[b]{0.28\textwidth}    
    \includegraphics[width=\textwidth]{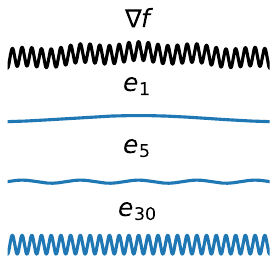}
    \caption{}
    \end{subfigure}
    \begin{subfigure}[b]{0.096\textwidth}    
    \includegraphics[width=\textwidth]{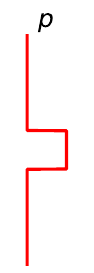}
    \caption{}
    \end{subfigure}
    \begin{subfigure}[b]{0.28\textwidth}
    \includegraphics[width=\textwidth]{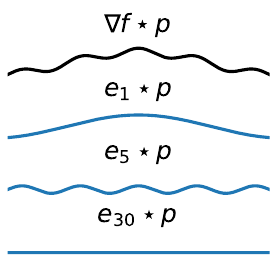}
    \caption{}
    \end{subfigure}
    \caption{A 1D visualization of amplification (b) and mitigation (d) effects caused by gradient and perturbation, respectively. For simplicity, we have shown a rectangular pulse (c) convolved with $\nabla f$, but the kernel shape is usually more complex in practice, an example shown in \cref{fig:bands}. Moreover, functions in blue are the Fourier representation of the black functions. Due to the gradient operator, small high-frequency jitters in the decision boundary (a) amplify proportionally to their frequency (b). To remove the effect of high-frequency components of the decision boundary that usually dominate the explanations, researchers have proposed different perturbation kernels, visualized in (d). \cref{prop:amplification} corresponds to (a~$\rightarrow$~b) and \cref{prop:mitigation} corresponds to (b~$\rightarrow$~d).}
    \label{fig:concept}
\end{figure}

When generating explanations, one is tempted to look for hyperparameters that make the explanation reflect what they \textit{think it should} \cite{chen_true_2020}. 
However, monitoring explanations from only one frequency band is potentially misleading and is likely to contradict itself simply by changing hyperparameters (see \cref{fig:frequency-explanations}) \cite{kindermans_reliability_2017,harel_inherent_2022}. We formalize this statement, also known as the Rashomon effect (see \cref{proofs} for the proof). 

\begin{proposition}
    The SmoothGrad-Squared pixel attribution can contradict itself by changing the hyperparameters. Simply,
    \textbf{the Rashomon effect can occur in the SmoothGrad-Squared explanations}.
\end{proposition}

\paragraph{Intuition.} This proposition suggests a reason, \ie  removing different frequency bands, for inconsistencies we encounter in SG$^2$ explanations, which has been unknown so far \cite{rong_consistent_2022}.
Note that this occurs with different kernels, but we have chosen the one that leads to a simpler proof.

\subsubsection{IntegratedGrad via a Rect Kernel.}
\label{rect-kernel-analysis}
Setting the perturbation kernel to $p(\tbsx) = \frac{1}{\sigma|\bss-\bsx|}\operatorname{Rect}(\frac{\tbsx-\bsx}{\sigma(\bss-\bsx)}-\frac{1}{2})$\footnote{Rect function is defined as $\operatorname{Rect}(\tbsx)=\frac{1}{2}\operatorname{sign}(\frac{1}{2}-|\tbsx|)+\frac{1}{2}$},
where $\bss$ is a chosen baseline image, 
we recover the perturbation kernel of a well known method called \textbf{IntegratedGrad} (IG) \cite{sundararajan_axiomatic_2017}.
We use the spectral representation of gradients in \cref{eq:general-spectral-representation-gradients}, for a proposition about the gradient signs.
\begin{align}
    \E_{p(\tbsx)}[\nabla f(\tbsx)] &\propto \int i\hf(\bsomega) \sin(\pi\sigma(\bss-\bsx)\bsomega)e^{i\pi\bsomega \bs{r}} d\bsomega
    \label{eq:spectral-representation-rect}
\end{align}
Where $\bs{r}=(1+2\frac{\bsx}{\sigma(\bss-\bsx)})$, appears because the kernel is not centered on the sample being explained. 
Note that the Fourier operator linearly expands along each dimension, but we use vector forms for notational simplicity.
Furthermore, to study the effect of elementwise multiplication on the explanation quantitatively, we use \textbf{IG} to denote integrated gradients without the elementwise multiplication with input and \textbf{XIG} to denote the version with elementwise multiplication.

In explanation visualizations, the negative gradient sign has been used to show a negative contribution to the prediction \cite{bach_pixel-wise_2015,shrikumar_not_2017}, to our knowledge it has never been formally justified. In the next proposition, we state that the negative gradient sign can be the direct result of the chosen perturbation kernel, formally shown in \cref{ig-sign-appendix-proof}.

\begin{proposition}
    The gradient sign in IG is affected by the perturbation level. 
    \label{ig-sign}
\end{proposition}

This proposition entails that the negative sign in the gradient explainer does not necessarily reflect a negative contribution to the prediction, nor the positive sign necessarily reflects a positive contribution (see \cref{fig:example-ig-row,fig:example-sg-row}).

\begin{remark}
    Since we assume continuity wrt $\sigma$, in our analysis, by Bolzano's theorem \cite{rudin_principles_1976}, there are hyperparameters that set the attribution to zero. Therefore, the gradient explainer may not produce a useful ranking compared to that of the gradient squared.  
\end{remark}

\begin{remark}
    \label{design-choice-squared}
    \cref{ig-sign} also shows that adding explanations with different hyperparameters may cause destructive interference when the different values of $\bsomega\hp(\bsomega)$ are out of phase. Hence, \textit{squared gradients is a more appropriate design choice}, unless the use of phase information is planned as part of the method. This has been discovered experimentally \cite{hooker_benchmark_2019} and confirmed in \cref{evaluation}.
\end{remark}

\section{Proposed Solutions Based on the Spectral Analysis}
\subsection{Finding an Optimal Perturbation Scale}
\label{kc-similarity-subsection}
Note that the squared gradients in spectral domain (\cref{eq:general-spectral-representation-sqgradients}) can be formulated as an inner product between the scaled PSD of the perturbation distribution and the PSD of the classifier \ie $\E_{p(\tbsx)}\left[(\nabla f(\tbsx))^2\right]= 4\pi^2 \langle \operatorname{S}_f , \|\bsomega\|^2\operatorname{S}_{\sqrt{p}} \rangle$. 
Assuming unimodality, a nonzero contribution, and that the norm $\left\| \|\bsomega\|^2\operatorname{S}_{\sqrt{p}} \right\|$ exists, then the cosine similarity is maximized at a certain perturbation scale that has the highest similarity with $\operatorname{S}_f$.
A higher cosine similarity guarantees recovering more information about the classifier.
\begin{definition}
    The cosine similarity between the PSD of the classifier and that of the perturbation kernel given by
    \begin{equation}
        \label{eq:kc-similarity}
        \text{Similarity}(\operatorname{S}_f,\|\bsomega\|^2\operatorname{S}_{\sqrt{p}}) = 4\pi^2 \frac{\langle \operatorname{S}_f , \|\bsomega\|^2\operatorname{S}_{\sqrt{p}}\rangle}{\left\| \operatorname{S}_{f} \right\| \left\| \|\bsomega\|^2\operatorname{S}_{\sqrt{p}} \right\|}
    \end{equation}
\end{definition}
This quantity can be used as a way to determine the optimal perturbation scale per pixel to maximize the information extracted with the squared gradients.
Note that, the word ``optimal'' refers to extracting the most information by the inner product, \textit{whether deemed relevant by a human or not.}.

For one of the perturbation kernels we have presented earlier in \cref{sect:guassian-kernel}, we have optimized this quantity on multiple levels, such as dataset level \cref{fig:kc-similarity-dataset}, image level \cref{tab:evaluations}, and, pixel level \cref{fig:frequency-explanations-d}. Finally, we have experimentally verified the effectiveness of this approach in \cref{evaluation}.

\subsubsection{Computing the Similarity of Classifier and Gaussian Kernel.} To compute the similarity, we first need to evaluate the kernel norm as follows
\begin{align}
    \|\|\bsomega\|^2 e^{-8\pi^2\sigma^2\|\bsomega\|^2}\|&= \left(\int  \|\bsomega\|^4 e^{-16\pi^2\sigma^2\|\bsomega\|^2} d\bsomega\right)^{\frac{1}{2}}\\
    &\propto \sigma^{\frac{-5}{2}}
\end{align}
Therefore, the cosine similarity for the Gaussian perturbation kernel can be computed as follows:
$$\text{Similarity}(\operatorname{S}_{\sqrt{p}},\operatorname{S}_f) \propto \sigma^{\frac{5}{2}}\langle \operatorname{S}_f , \|\bsomega\|^2\operatorname{S}_{\sqrt{p}}\rangle$$
Using the L1 norm for the squared gradients of flattened explanations. We have observed experimentally that $\operatorname{S}_{f}$ depends on the input scale, hence we use a convex interpolation with noise to support the assumption that $\operatorname{S}_{f}$ remains constant (see \cref{sec:signal-to-noise-ratio,fig:combination-fn}). The results are shown in \cref{fig:kc-similarity-dataset}.

\begin{remark}
    \label{design-choice-optimal-sigma}
    The optimal perturbation scale, determined by maximizing cosine similarity, offers a standard method for selecting the perturbation scale in SG$^2$. Given that user discretion in the selection of perturbation scale contributes to inconsistencies, as illustrated in \cref{fig:example-sgsq-row}, opting for the optimal perturbation scale mitigates these inconsistencies in the results. Hence, \textit{it represents a more suitable design choice}. In \cref{evaluation} we demonstrate that this approach not only enhances SG$^2$'s robustness against inconsistencies but also its adaptability to diverse datasets, which in turn can improve the quantitative performance.
\end{remark}

\begin{figure}
    \centering
    \begin{subfigure}[b]{0.49\textwidth}
    \includegraphics[width=\textwidth]{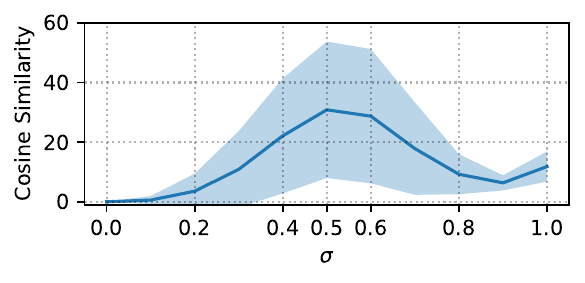}
    \caption{}
    \end{subfigure}
    \begin{subfigure}[b]{0.49\textwidth}
    \includegraphics[width=\textwidth]{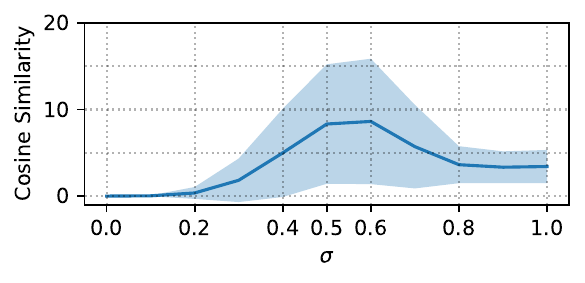}
    \caption{}
    \end{subfigure}
    \caption{Cosine similarity between kernel and classifier, defined in \cref{eq:kc-similarity}, is measured for the Gaussian perturbation kernel and ResNet50 averaged on the 10K images of the ImageNet validation set (a) and Food101 (b). This metric assesses the similarity between the PSD of the perturbation kernel and that of the classifier, indicating the perturbation level's effectiveness for kernels with finite norm in the Fourier domain. When using a Dirac explanation prior in \cref{eq:spectral-lens}, where a perturbation level is assigned per image or dataset, the optimal choice for the perturbation level is where the similarity is maximized across the image or dataset respectively. Note that this level of noise refers to a band-pass filter, shown in \cref{fig:bands}, which suggests the consideration of a latent distribution on the \textit{frequency of labeling}, reflecting the workforce's mindset during data labeling.}
    \label{fig:kc-similarity-dataset}
\end{figure}

\subsection{Aggregation of Information With Explanation Prior}
\label{aggregation-of-information}
We have shown that explanations are influenced by specific frequency bands, rather than representing the importance of individual pixels for prediction. Instead, \textit{they indicate the contributions within a specific frequency range}. This observation motivates us to investigate a method for aggregation of information across different bands, which closely resembles an ensemble of IG and SG \cite{sundararajan_axiomatic_2017,smilkov_smoothgrad_2017} which we call SpectralLens.

\begin{definition}
    \textbf{SpectralLens} is an ensemble explanation method, that aggregates information across different frequency bands as follows
    \begin{align}
    \E_{p(\sigma)}\left[\E_{p(\tbsx)}\left[(\nabla f(\tbsx))^2\right]\right]\propto \iint \operatorname{S}_f(\bsomega) \|\bsomega\|^2 e^{-8\pi^2\sigma\|\bsomega\|^2} d\bsomega p(\sigma) d\sigma
\label{eq:spectral-lens}
\end{align}
Where we refer to this method by SL$^2$ as it is using the squared gradients and we call $p(\sigma)$ as an explanation prior, gauging the frequency bands that we choose to represent in an explanation.
\end{definition}
 
\begin{remark}
    \label{design-choice-no-dirac-prior}
    SL$^2$ exhibits robustness against inconsistencies in SG$^2$ due to its aggregation of information from various bands. Therefore, \textit{it is a preferable design choice not to employ a Dirac explanation prior with SG$^2$}, as it eliminates inconsistencies, depicted in \cref{fig:example-sgsq-row}. Nevertheless, as demonstrated in the \cref{evaluation}, this approach may enhance the quantitative performance of explanations.
\end{remark}

\paragraph{Intuition.}
Considering \cref{eq:spectral-representation-gaussian}, we realize that the combination of perturbation and gradient explainer creates overlapping band-pass carrier waves that modulate the amplitude of the spectral density of the classifier. Therefore, a new summary statistic is required that can reflect a wide range of frequency bands.
We explicitly identified the spectral properties of the explanations, recognizing that not all frequencies are equally desired. Explanations with very high frequencies primarily attribute predictions to pixels. Therefore, we utilize an \textit{explanation prior} $p(\sigma)$ for a nuanced control of the spectral properties of explanations.
Although there are many possible alternatives for selecting an explanation prior, we used $p(\sigma)=\operatorname{Unif}(0,1)$ for simplicity.

\begin{definition}
    We define \textbf{ArgLens} (AL) as the point at which the SG$^2$ is maximized for pixel $x_i$. Formally defined as
    \begin{align}
        \label{eq:arglens}
        \operatorname{ArgLens}(x_i) = \arg\max_{\sigma} \left\{\operatorname{SG}^2(x_i)\right\}
    \end{align}
    where $\sigma$ is the hyperparameter of the perturbation kernel.
\end{definition}
AL shows the frequency band at which the pixel's attribution is maximized. Alternatively, it can be interpreted as the robustness of a pixel to noise.
We have visualized the AL for the Gaussian kernel in \cref{fig:frequency-explanations} and discussed the way it helps the interpretation of explanations in \cref{interesting-properties-of-sl-al}.

The aggregation of explanations from different frequencies is visualized in \cref{fig:frequency-explanations}. Moreover, the basic explanation methods are obtainable with a Dirac explanation prior $p(\sigma) = \delta(\sigma)$. Therefore, it can be used as a unified formulation for gradient-based explanation methods in spectral domain.

\subsubsection{Interesting Properties of SpectralLens and ArgLens.}
\label{interesting-properties-of-sl-al}
In the absence of a robust mathematical framework, the interpretation of explanations may be subjective. Thus, explanation methods must be grounded in solid theoretical principles. SL$^2$ addresses potential inconsistencies by aggregating attributions across frequency bands, ensuring fair representation of all pixels' contributions.

As SG$^2$ results in bands-pass filters (see \cref{sect:guassian-kernel}), its attributions reflect the changes in the decision boundary within a certain frequency band. 
Furthermore, AL reveals the frequency band at which the attribution is maximized. 
Hence, lower AL indicates that the attributions are likely governed by high-frequency components of the decision boundary, more sensitive to input noise, consequently less robust. 
An ideal classifier with smooth decision boundaries, \ie robust features, leads to high AL everywhere, with a high intensity in SL$^2$ on relevant pixels and low intensity on irrelevant pixels.

Given the different purposes of pixels within an image, the frequency of activation may vary and can be visualized via AL. Since SL$^2$ and AL offer complementary information, they should be interpreted together. This approach offers a deeper understanding compared to a simplistic perspective, as it has been shown by Rahaman \etal~\cite{rahaman_spectral_2019} that high-frequency regions emerge later in training, are more challenging to fit, and are prone to change. On the contrary, low-frequency components of the decision boundary manifest earlier in training, are easier to fit, and exhibit greater stability with respect to noise.

It is important to clarify that ``high-frequency'' refers to features considered high-frequency by the model, and not necessarily frequency across spatial dimensions (\eg width, height). 
Nonetheless, empirical findings demonstrate that high-frequency features in spatial domain are more likely to be identified as high frequency by the model (see \cref{appendix:example-outputs}).

\begin{figure}[t]
    \centering
    \begin{subfigure}[b]{0.24\linewidth}
    \includegraphics[width=\textwidth]{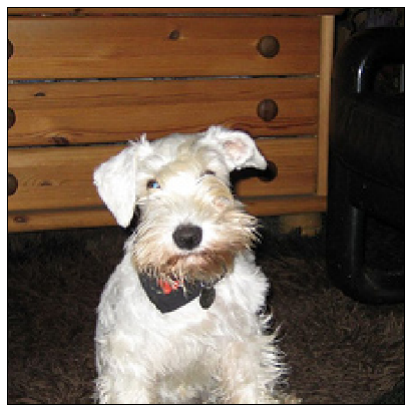}\\
    \includegraphics[width=\textwidth]{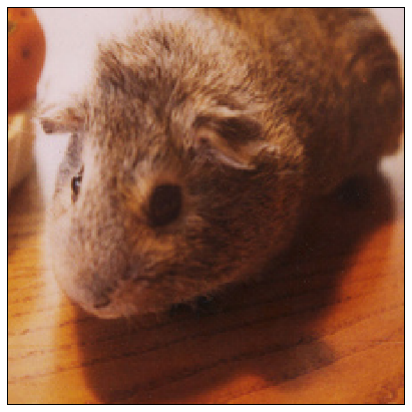}
    \caption{}
    \end{subfigure}
    \begin{subfigure}[b]{0.24\linewidth}
    \includegraphics[width=\textwidth]{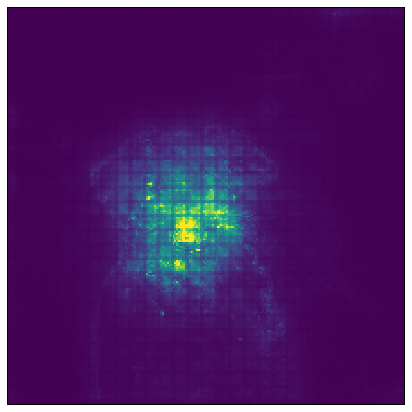}\\
    \includegraphics[width=\textwidth]{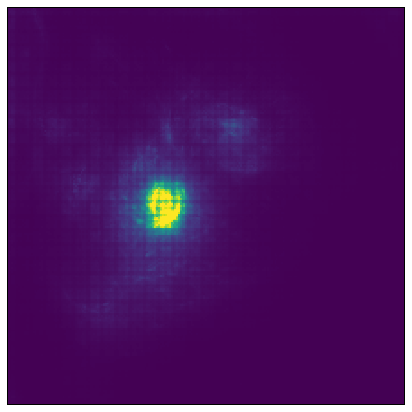}
    \caption{}
    \end{subfigure}
    \begin{subfigure}[b]{0.24\linewidth}
    \includegraphics[width=\textwidth]{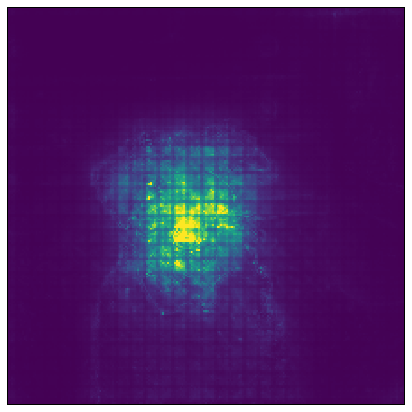}
    \includegraphics[width=\textwidth]{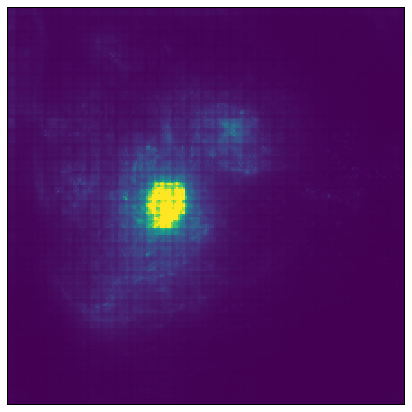}
    \caption{}
    \end{subfigure}
    \begin{subfigure}[b]{0.24\linewidth}
    \includegraphics[width=\textwidth]{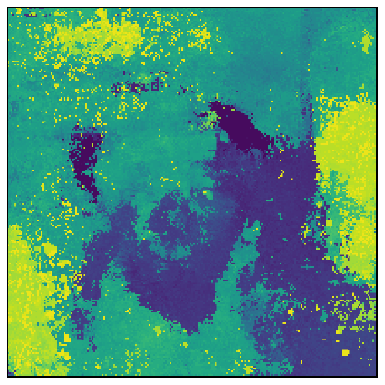}\\
    \includegraphics[width=\textwidth]{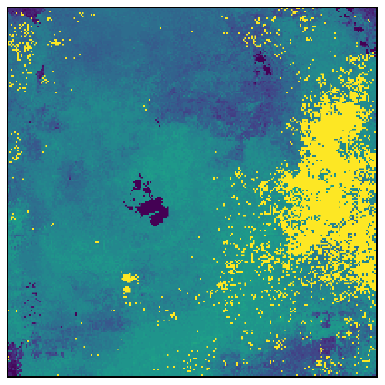}
    \caption{}
    \label{fig:frequency-explanations-d}
    \end{subfigure}
    \caption{This figure shows example visualizations of explanations generated by SG$^2$ implementing design decisions suggested in our work in \cref{design-choice-optimal-sigma,design-choice-no-dirac-prior}. (b) shows the SG$^2$ visualized at the maximum similarity according to \cref{fig:kc-similarity-dataset}, for this image $\sigma=0.5$. 
    (c) shows the SL attributions. 
    Lastly, (d) shows the frequencies obtained from AL (defined in \cref{eq:arglens}), which correspond to the perturbation level at which SG$^2$ is maximized. Since we are interested in frequencies rather than the actual perturbation levels, we visualize 2 stds above and below the mean of frequencies obtained by $\omega = 1/(1+\text{AL})$ (see \cref{appendix:example-outputs} for more examples). The heatmaps in (b) and (c) show how much the model is relying on the pixel, while (d) shows the frequency at which the feature's contribution maximizes. 
    In the provided examples, we can see that the model is relying on the nose (top) and eye (bottom) while the same region is of a lower frequency compared to the head. This can be interpreted as the model finding the nose and eye more stable features compared to the head. 
    }
    \label{fig:frequency-explanations}
\end{figure}

\section{Evaluation}
\label{evaluation}
Given the experimental discoveries that squared gradients lead to better results \cite{hooker_benchmark_2019}, we evaluate both scenarios for different hyperparameters and compare them with SL in \cref{tab:evaluations,tab:vit}, refer to \cref{reproducibility} for implementation details and reproducibility.
For the purposes of this study, we use pixel removal strategy, despite its problems \cite{hooker_benchmark_2019}. The spectral analysis of pixel removal strategy and possible solutions are deferred to future research.

We report single-pixel most-important-first (MoRF) insertion and deletion scores of the methods to assess the design choices previously proposed by our theoretical framework: (1) utilizing squared gradients instead of gradients in \cref{design-choice-squared}, (2) finding a standard perturbation scale in \cref{design-choice-optimal-sigma}, and (3) employing a uniform explanation prior instead of Dirac, stated in \cref{design-choice-no-dirac-prior} (see \cref{fig:imagenet-deletion-insertion-scores} for an ablation study of kernel variance and \cref{sec:ablation-erosion} for an ablation study of patch removal instead of single pixel removal).
We evaluated the suggested design choices on ImageNet\cite{deng_imagenet_2009} and Food101\cite{noauthor_food-101_nodate} in \cref{tab:evaluations}.
However, smaller datasets, such as CIFAR, are omitted because their resolution restricts the potential gap between high- and low-frequency features.

\begin{table}[t!]
\caption{
This table presents the assessment of the proposed design choices in \cref{design-choice-squared,design-choice-optimal-sigma,design-choice-no-dirac-prior} across different methods, highlighting statistically significant improvements with bold fonts.
We conduct t-tests comparing gradient methods to their squared-gradient counterparts to evaluate the design choice in \cref{design-choice-squared}. Furthermore, we conduct t-tests between SG$^2$ with the original heuristic and those ensuring consistency to evaluate design decisions in \cref{design-choice-optimal-sigma,design-choice-no-dirac-prior}. Overlines and underlines denote highest or lowest values within columns, respectively.
Although there are no squared versions of Occlusion Sensitivity (OS) \cite{ancona_towards_2018} and RISE \cite{petsiuk_rise_2018}, they are included as low-frequency baselines. 
While the dimensions of boolean masks utilized in OS and RISE can implicitly regulate the sampling variance (see \cref{appendix:unifying}), the exponential increase in the required number of samples in high-frequency explanations renders them less practical.
SG scores follow the heuristic of the SmoothGrad paper \cite{smilkov_smoothgrad_2017} \ie $\sigma=\left(\max(\bsx)-\min(\bsx)\right)0.1$. 
XIG indicates the input multiplication in the original implementation of IntegratedGrad \cite{sundararajan_axiomatic_2017}, while IG does not apply such transformation. 
Rankings for methods using gradients are obtained from negative gradients, including RISE and Occlusion Sensitivity, according to \cref{eq:general-spectral-representation-gradients}.
Furthermore, SG$^2_{\text{Opt.}}$ uses an optimized kernel per image and SL$^2$ uses a uniform explanation prior.
Comparing SG$^2$ and SG$^2_{\text{Opt.}}$, we can see that SG$^2_{\text{Opt.}}$ can adapt to the statistics of the Food101. 
Therefore, SG$^2_{\text{Opt.}}$, while not subject to inconsistencies (shown in \cref{fig:example-sgsq-row}), can improve the performance due to its adaptability. 
Hence, we can ensure the consistency of SG$^2$ without compromising its performance.
Also, note that the input multiplication enhances visual quality, but degrades the performance.
}
\label{tab:evaluations}
\centering
\begin{tabular}{llcc|cc} 
\toprule
&& \multicolumn{2}{c}{ImageNet} & \multicolumn{2}{c}{Food101} 
\\ 
\cmidrule(lr){3-4}
\cmidrule(lr){5-6}
& & Insertion($\uparrow$) & Deletion($\downarrow$) & Insertion($\uparrow$) & Deletion($\downarrow$) 
\\ 
\cmidrule(lr){3-6}
\parbox[t]{5mm}{\multirow{5}{*}{\rotatebox[origin=c]{90}{Gradients}}} 
& OS & 
$0.384\pm0.008$ & $0.237\pm0.010$ 
& $0.378\pm0.011$ & $0.209\pm0.009$ 
\\
& RISE & 
$\overline{0.494}\pm0.014$ & $0.174\pm0.007$ 
& $\overline{0.410} \pm0.013$ & $0.198\pm0.007$ 
\\ 
& SG &
$0.212\pm0.010$ & $0.179\pm0.009$ 
& $0.179\pm0.008$ & $\underline{0.137}\pm0.007$ 
\\
& XIG &
$0.208\pm0.009$ & $0.204\pm0.008$ 
& $0.101\pm0.007$ & $0.177\pm0.007$ 
\\ 
& IG & 
$0.201\pm0.010$ & $0.182\pm0.009$ 
& $0.145\pm0.006$ & $\underline{0.131}\pm0.005$ 
\\  
\midrule
\parbox[t]{5mm}{\multirow{5}{*}{\rotatebox[origin=c]{90}{Gradients$^2$}}}
& SG$^2$ & 
$\mathbf{0.465}\pm0.013$ & $\underline{\mathbf{0.157}}\pm0.007$ 
& $\mathbf{0.371}\pm0.012$ & $0.162\pm0.006$ 
\\
& XIG$^2$ &
$\mathbf{0.334}\pm0.014$ & $\mathbf{0.187}\pm0.008$ 
& $\mathbf{0.204}\pm0.010$ & $0.263\pm0.007$ 
\\ 
& IG$^2$ & 
$\mathbf{0.340}\pm0.012$ & $\mathbf{0.168}\pm0.007$ 
& $\mathbf{0.320}\pm0.009$ & $\underline{0.135}\pm0.005$ 
\\ 
\cmidrule{2-6}
& $\text{SG}^2_{\text{Opt.}}$ &
$\overline{\mathbf{0.491}}\pm0.015$ & $0.187\pm0.007$ 
& $\mathbf{0.394}\pm0.013$ & $0.167\pm0.006$ 
\\
& SL$^2$ & 
$\overline{\mathbf{0.493}}\pm0.016$ & $0.169\pm0.007$ 
& $\mathbf{0.395}\pm0.014$ & $0.157\pm0.007$
\\
\bottomrule
\end{tabular}
\end{table}

\subsubsection{Limitations of Our Framework.} 
The choice of independent Gaussian perturbation is limited due to its inability to incorporate prior knowledge about the input space or data manifold characteristics, and other data modalities such as graphs and time series. 
However, preferring simpler analytical forms within our framework limits the alternative choices. 
Efficiency is another challenge compared to basic methods, as performing back-props at various noise scales is costly. 
Moreover, our attribution of frequencies in visualizations, such as SL, may lack the desired generalizability across data points, as high-frequency regions in one image may correspond to low-frequency regions in another image (see \cref{appendix:example-outputs}). 
Identifying generalizable attributions across images remains a subject for future investigation.

\section{Related Works}
\label{related_works}
\subsubsection{Probabilistic and Spectral View of Explanations.} 
Despite sporadic attempts to analyze explanations from a probabilistic view \cite{chen_learning_2018,wang_probabilistic_2021}, this view is more common in the literature around the uncertainty quantification of explanations \cite{slack_reliable_2021,hill_explanation_2022,marx_but_2023}.
Our work is closely related to Covert \etal~\cite{covert_explaining_2022}, if their work is seen from a probabilistic perspective.
Aside from the works around the spectral view of deep networks \cite{rahaman_spectral_2019,tancik_fourier_2020,basri_frequency_2020,cao_towards_2021,benbarka_seeing_2022}, as of our current knowledge, there is no work adopting a spectral view of explanations along pixels' value.
Nonetheless, there are works that build on a spectral view along spatial dimensions \cite{kolek_explaining_2023,kolek_cartoon_2022}, extending an optimization-based definition of explanations \cite{fong_understanding_2019}.

\subsubsection{Unifying Views to Explanations.}
Covert \etal \cite{covert_explaining_2022} propose an information theoretic framework that unifies several explanation methods, focusing on three design choices: (1) feature removal method, (2) which model behavior to analyze, and (3) summary statistics.
Our probabilistic framework, being similar to Covert's, offers a distinct viewpoint relevant to our spectral analysis. 
It can be used as a unifying view of explanations (see \cref{appendix:unifying}), inheriting from Covert's framework.
Notably, our work aligns with the broader context of research under finding unifying representations of explanations \cite{agarwal_towards_2021,lundberg_unified_2017,han_which_2022}.

\subsubsection{Input-space or Function-space Perturbations.}
Methods such as NoiseGrad \cite{bykov_noisegrad_2022}, Grad-CAM \cite{selvaraju_grad-cam_2020}, RectifiedGrad \cite{kim_why_2019}, Layer-wise Relevance Propagation \cite{bach_pixel-wise_2015}, DeepLift \cite{shrikumar_not_2017}, and even Sharpness-Aware Minimization (SAM) \cite{foret_sharpness-aware_2021} employ function space perturbations. Utilizing such distributions offers data modality independence, whereas input space perturbation remains agnostic about the model architecture. While acknowledging the duality between function and input space distributions, our focus on explicit input space distributions excludes these methods from our study. Although we recognize the utility of integrating function space perturbation into our mathematical model, we defer this aspect to future investigations.
It can be noted, however, that extending this work to function space perturbation will connect it to the literature on double descent and smooth interpolation of samples \cite{nakkiran_deep_2019,gamba_lipschitz_2023}.

\section{Conclusion}
In this study, we deepen understanding of pixel attribution methodologies by rigorously analyzing their theoretical underpinnings, revealing a common spectral bias. Our analysis supports the experimental finding of recent works regarding the efficacy of using squared gradients. Moreover, we demonstrate how spectral properties can lead to inconsistent explanations and propose two design choices informed by our theoretical framework to address these inconsistencies: standardizing hyperparameters based on cosine similarity between the perturbation kernel and classifier, and aggregating explanations with a given prior. Despite the effectiveness of our approach, we acknowledge limitations and suggest avenues for further research. Lastly, we assess the suggested design choices using deletion and insertion scores.
\clearpage

\section*{Acknowledgements}
We would like to thank \href{https://www.kth.se/profile/glma}{Giovanni Luca Marchetti} for an early review of our work and his kind feedback. This project is partially supported by Region Stockholm through MedTechLabs, and  \href{https://wasp-sweden.org/}{Wallenberg AI, Autonomous Systems and Software Program (WASP)} funded by the Knut and Alice Wallenberg Foundation. Scientific computation was enabled by the supercomputing resource Berzelius, provided by the National Supercomputer Centre at Linköping University and the Knut and Alice Wallenberg foundation. 
 
%
%
\bibliographystyle{splncs04}
\bibliography{references}
\newpage
\appendix
\section{A Spectral Representation of Explanations}
\label{appendix:derivations}
In this section, we derive spectral representation for explanations analogous to analysis of a linear time-invariant systems in spectral signal processing, and separate the analysis by the power of the gradients.
In \cref{sect:guassian-kernel}, we have discussed the emergence of band-pass filters that cause the explanations to be inconsistent given different perturbation scales. 

\subsection{Spectral Representation of Prediction}
\label{spectral-representation-prediction}
To analyze explanations from the spectral viewpoint, we first need a version of $f$ that acts on a signal $p(\tbsx)$, which we denote by $F(p(\tbsx))$.
We assume linearity of response of $F$ with respect to the input signal, which lets us denote the response of the signal by $\int F(\delta_{\tbsx}(\bsx)) p(\tbsx) d \tbsx$.
The linearity of response of a deep network is indeed a strong assumption, and it should be pointed out that the validity of this assumption is tightly coupled with the concentration of the perturbation distribution around the vicinity of the sample being explained.
As we are analysing the networks in small neighborhoods of a sample we assume that the Fourier transforms exist and our function has a linear response with respect to its signal.
Finally, assuming that the impulse response can be obtained by a forward pass \ie $f(\tbsx) = F(\delta_{\tbsx}(\bsx))$, we can write the response of the system $F$ to the signal $p(\tbsx)$ by
\begin{align}
    F(p(\tbsx)) &= \int F(\delta_{\tbsx}(\bsx)) p(\tbsx) d \tbsx\\
    &= \int f(\tbsx) p(\tbsx) d \tbsx \\
    &= \int \hf(\bsomega) \widehat{p}(\bsomega) d \bsomega
\end{align}
where the last equality follows from the Plancherel theorem \cite{reiter_classical_2000}.
The result of the application of a model to the signal (or perturbation distribution) can be written as a convolution $f*p$, therefore, one may use the ``perturbation kernel'' and ``perturbation distribution'' interchangeably.
This representation helps us to perform a spectral analysis of the explanations presented in the next section.

\subsection{Gradients in Spectral Domain}
\label{appendix:derivations-gradient}
To find the spectral representation of gradients, we use the Plancherel theorem \cite{reiter_classical_2000}, for $\E_{p(\tbsx)}$ as follows
\begin{align}
\E_{p(\tbsx)}[\nabla f] & = \int \nabla f(\tbsx) p(\tbsx)d \tbsx\\
 &=\int i2\pi\bsomega \hf(\bsomega) \hp(\bsomega) d \bsomega
\end{align}
note that the actual pixel values cause a shift in Fourier domain, which can be ignored when assuming the origin is set at the sample being explained. 
Let us denote the real even and imaginary odd components  of a function $\hf$ by $\hf_E$ and, $\hf_O$ respectively (note that as we are using real valued functions, other components are zero). 
Then we can further simplify the last integral and gain interesting insights about the gradients and the information each perturbation kernel extracts
\begin{align}
\E_{p(\tbsx)}[\nabla f] &=\int i2\pi\bsomega \hf(\bsomega) \hp(\bsomega) d \bsomega\\
=& 2\pi\left[\int i\bsomega (\hf_E(\bsomega)+i\hf_O(\bsomega)) (\hp_E(\bsomega)+i\hp_O(\bsomega)) d \bsomega\right]\\
=& 2i\pi\underbrace{\left[\int \bsomega\hf_E(\bsomega)\hp_E(\bsomega)d\bsomega - \int \bsomega\hf_O(\bsomega)\hp_O(\bsomega) d \bsomega\right]}_{\text{odd functions over symmetric bounds} = 0}\\
&-2\pi\left[\int \bsomega\hf_E(\bsomega)\hp_O(\bsomega) d \bsomega + \int \bsomega\hf_O(\bsomega)\hp_E(\bsomega) d \bsomega \right]\\
=& -2\pi\left[\int \bsomega\hf_E(\bsomega)\hp_O(\bsomega) d \bsomega + \int \bsomega\hf_O(\bsomega)\hp_E(\bsomega) d \bsomega\right]\\
=& -2\pi\left[\int \bsomega(\hf_E(\bsomega)\hp_O(\bsomega) + \hf_O(\bsomega)\hp_E(\bsomega)) d \bsomega\right]
\end{align}

\subsection{Squared Gradients in Spectral Domain}
\label{appendix:derivations-sq-gradient}
Let function $g(\bsx) = \int \nabla f(\tbsx) \sqrt{p(\bsx-\tbsx)}d \tbsx$,
expanding the autocorrelation function $\operatorname{R}_{g}(\bstau)$ for $g$, we have
\begin{align}
    \operatorname{R}_{g}(\bstau) &= \int g(\bsx+\bstau) \overline{g(\bsx)} d \bsx \\
    &= \iint \nabla f(\tbsx_1) \sqrt{p(\bsx+\bstau-\tbsx_1)} d \tbsx_1  \overline{\int \nabla f(\tbsx_2) \sqrt{p(\bsx-\tbsx_2)} d \tbsx_2} d\bsx\\
    &= \iint \nabla f(\tbsx_1) \overline{\nabla f(\tbsx_2)} \left(\int \sqrt{p(\bsx+\bstau-\tbsx_1)} \overline{\sqrt{p(\bsx-\tbsx_2)}} d\bsx \right) d \tbsx_1 d \tbsx_2\\
    &= \iint \nabla f(\tbsx_1) \overline{\nabla f(\tbsx_2)} \operatorname{R}_{\sqrt{p}}(\tbsx_1-\tbsx_2+\bstau) d \tbsx_1 d \tbsx_2
\end{align}
where overbars show complex conjugate.
This can establish a connection between the autocorrelation of $g(\bsx)$ and that of the perturbation kernel $p(\tbsx - \bsx)$. 
Moreover, from the Wiener-Khinchin theorem \cite{cohen_generalization_1998}, we know that $\operatorname{S}_g(\bsomega) = \widehat{\operatorname{R}}_g(\bsomega)$, i.e. the Fourier transform of autocorrelation function of $g$ is equal to its spectral density.
By further expanding the expression, one can establish a connection between the spectral density of $g(\bsx)$ and $p(\tbsx - \bsx)$, as follows:
\begin{align}
    \operatorname{S}_{g}(\bsomega) &= \int \operatorname{R}_{g}(\bstau) \psi(\bsomega,\bstau) d\bstau\\
    &=\iiint \nabla f(\tbsx_1) \overline{\nabla f(\tbsx_2)} \operatorname{R}_{\sqrt{p}}(\tbsx_1-\tbsx_2+\bstau) \psi(\bsomega,\bstau) d \tbsx_1 d \tbsx_2 d\bstau\\
    &=\iint \nabla f(\tbsx_1) \overline{\nabla f(\tbsx_2)} \left(\int\operatorname{R}_{\sqrt{p}}(\tbsx_1-\tbsx_2+\bstau) \psi(\bsomega,\bstau) d\bstau\right) d \tbsx_1 d \tbsx_2\\
    &=\iint \nabla f(\tbsx_1) \overline{\nabla f(\tbsx_2)} \left( \psi(\bsomega,\tbsx_1-\tbsx_2) \operatorname{S}_{\sqrt{p}}(\bsomega)\right) d \tbsx_1 d \tbsx_2\\
    &=\left(\int \nabla f(\tbsx_1)\psi(\bsomega,\tbsx_1)d \tbsx_1\right) \overline{\left(\int \nabla f(\tbsx_2)\psi(\bsomega,\tbsx_2) d \tbsx_2\right)} \operatorname{S}_{\sqrt{p}}(\bsomega)\\
    &=\left|i2\pi\bsomega \hat{f}(\bsomega)\right|^2 \operatorname{S}_{\sqrt{p}}(\bsomega)\\
    &=4\pi^2 \|\bsomega\|^2 \operatorname{S}_f(\bsomega) \operatorname{S}_{\sqrt{p}}(\bsomega)
\end{align}
The major difference with the conventional spectral analysis of signals \cite{stoica_spectral_2005}, is the term multiplied by the spectra of $f$ which is the result of the gradient operator. Finally, one can write the expected value of squared gradients in spectral domain as
\begin{align}
    \E_{p(\tbsx)}\left[(\nabla f(\tbsx))^2\right]&= \int (\nabla f(\tbsx))^2 \left(\sqrt{p(\tbsx)}\right)^2 d \tbsx\\
    &=\operatorname{R}_{g}(0)\\
    &=\int \operatorname{S}_g(\bsomega) \psi(-\bsomega,0) d \bsomega\\
    &=\int \operatorname{S}_g(\bsomega) d \bsomega\\
    &=\int 4\pi^2 \|\bsomega\|^2 \operatorname{S}_f(\bsomega) \operatorname{S}_{\sqrt{p}}(\bsomega) d\bsomega
\end{align}
A counterintuitive result here is that any Dirac perturbation kernel would have the same form, which is the result of assuming that the origin is centered around the example being explained.

\section{Theorems and Proofs}
\setcounter{theorem}{0}
\setcounter{proposition}{0}
\label{proofs}
This section contains a more formal approach towards the theorems and propositions stated in this work.
\begin{theorem}
    Any perturbation kernel $p(\tbsx)$ mitigates the attribution of high-frequency features; simply put, \textbf{any perturbation is a low-pass filter}.
\end{theorem}
\begin{proof}
    \label{prop:mitigation-proof}
    Let $\hp(\bsomega)$ denote the FT of the absolutely integrable perturbation kernel $p(\tbsx)\in L^1$. It follows directly from the Riemann-Lebesgue lemma \cite{serov_riemannlebesgue_2017}, that for every $M\in(0,1)$\footnote{The $\|\hp(\bsomega)\|\leq 1$ inequality follows from the properties of the characteristic function of any probability distribution.} there exists a $\bsomega^*$ that $\forall \|\bsomega\|~>~\|\bsomega^*\|$ we have
    $\|\hp(\bsomega)\|~<~M$.
\end{proof}

\begin{theorem}
    \label{prop:amplification-appendix}
    The gradient operator amplifies the attribution of high-frequency features; simply put, \textbf{the gradient operator is a high-pass filter}.
\end{theorem}
\begin{proof}
    \label{prop:amplification-proof}
    Given $M>0$, let $\bsomega^* = \sqrt{M}\mathbf{1}$. Then for any $ \|\bsomega\|<\|\bsomega^*\|$ we have $\|\bsomega\|<M$, which shows that gradient acts as a high-pass filter in spectral domain.
\end{proof}

\begin{proposition}
    \label{rashomon-effect-appendix}
    The SmoothGrad-Squared pixel attribution can contradict itself by changing the hyperparameters. Simply,
    \textbf{the Rashomon effect can occur in the SmoothGrad-Squared explanations}.
\end{proposition}
\begin{proof}
    \label{rashomon-effect-proof}
    For simplicity, assume that the total contribution of all pixels to a prediction in the image $\bsx$ is equal. 
    Formally, we have $\forall i;\int\hf_i(\omega)^2 d\omega=1$, where $\hf_i(\omega)^2$ represents the contribution of the pixel $x_i$ to the prediction at frequency $\omega$. 
    Furthermore, assume that the features, although having a total equal contribution, would contribute at different frequencies. 
    For simplicity, we assume that $\hf_i(\omega)^2 = \delta_{\omega_i}(\omega)$ for $\omega\geq0$. 
    Consequently, following \cref{eq:spectral-representation-gaussian}, the contribution of pixel $x_i$ at noise level $\sigma$ is given by
    \begin{align}
        \restr{\E_{p(\tbsx)}\left[(\nabla f(\tbsx))^2\right]}{x_i}
        &\propto \int \omega^2 \delta_{\omega_i}(\omega) e^{-8\pi^2\sigma^2\omega^2} d\omega\\
        &= \omega_i^2 e^{-8\pi^2\sigma^2\omega_i^2}
    \end{align}
    For brevity, denote $\restr{\E_{p(\tbsx)}\left[(\nabla f(\tbsx))^2\right]}{x_i}=\operatorname{A}_\sigma(x_i)$, and, without loss of generality, assume $|\omega_i|<|\omega_j|$ for some $i$ and $j$. It can be seen that the attribution of the high-frequency feature is higher, \ie $\operatorname{A}_\sigma(x_i)<\operatorname{A}_\sigma(x_j)$ if $\sigma~<~|\Bar{\omega}|$, 
    and lower, \ie $\operatorname{A}_\sigma(x_i)>\operatorname{A}_\sigma(x_j)$ if $\sigma~>~|\Bar{\omega}|$, where $|\Bar{\omega}|\propto\frac{\ln(|\omega_j|)-\ln(|\omega_i|)}{\omega_j^2-\omega_i^2}$. 
    
    This shows that a pixel that does not appear to be important with one hyperparameter may turn out to be important with another hyperparameter, therefore, leading to inconsistent explanations.
\end{proof}

\begin{proposition}
    The gradient sign in IG is affected by the perturbation level. 
    \label{ig-sign-appendix}
\end{proposition}
\begin{proof}
    \label{ig-sign-appendix-proof}
    For simplicity, assume that the total contribution of all pixels to a prediction in the image $\bsx$ are equal and even. 
    Formally, assume $\forall i;\int \hf_i(\omega) d\omega=1$, where $\hf_i(\omega)$ represents the contribution of pixel $x_i$ to the prediction at frequency $\omega$. 
    Furthermore, for simplicity, let us assume that $\hf_i(\omega) = \frac{1}{2}\delta_{\omega_i}(|\omega|)$ where $\omega_i>0$. 
    Consequently, following \cref{eq:spectral-representation-rect}, the contribution of pixel $x_i$ at perturbation level $\sigma$ is given by
    \begin{align}
        \restr{\E_{p(\tbsx)}[\nabla f(\tbsx)]}{x_i} &\propto \int \frac{1}{2} \delta_{\omega_i}(|\omega|) \sin(\pi(s_i-x_i)\sigma\omega) e^{i\pi\omega r_i} d\omega\\
        &= -\sin(\pi(s_i-x_i)\sigma\omega_i)\sin(\pi\omega_i r_i)
    \end{align}
    This shows that, depending on the value of $\sigma$ in the final term, the pixel attributions can be negative or positive, although $\hf_i(\omega)>0$ for all $\omega$.
\end{proof}

Note that in the proofs above, we have assumed the rect kernel on a single pixel for simplicity. Therefore, it is not meant to be advocated for measuring truthfulness on a single pixel.
Nonetheless, the change of sign can occur with other kernels, also on many pixels (see \cref{fig:example-sg-row})
\section{On the Generality of the Probabilistic View}
\label{appendix:unifying}
We have selected a few perturbation-based explanation methods to translate into our formulation, whose Fourier transform is easier to analyze.
Without loss of generality, we may use scalar notation for simplicity.

\begin{itemize}
    \item \textbf{Dirac Kernel:}
    A baseline, known as \textbf{VanillaGrad}, to which many have compared their work, is when explainer is set to $\nabla$ and the perturbation distribution is set to $p(\tbsx) = \delta_{\bsx}$. 
    Moreover, \textbf{Input$*$Grad} \cite{shrikumar_not_2017}, is the case where the explainer is set to $\bsx\nabla$.
    In this case, the PSD of the perturbation distribution is $|S_{\sqrt{p}}(\bsomega)|^2 = \|2\pi\bsomega\|^2$, which recovers \textit{the scaled spectrum} of $\hf$, which causes amplification of high-frequency features of the input signal. 
    
    \item \textbf{Dirac Comb Kernel:} In \textbf{RISE} \cite{petsiuk_rise_2018} and \textbf{Occlusion Sensitivity} \cite{ancona_towards_2018}, a low-resolution boolean mask of size $M \ll N$, is sampled from a distribution, which is set to a set of Bernoulli trials $\bsalpha\sim\{\operatorname{Ber}(h_i)\}_{i\in [M]}$ in RISE and a one-hot categorical $\bsalpha\sim \operatorname{OHCat}(H,M)$ in Occlusion Sensitivity, where $h_i$ is the $i$'th element of the success probability vector $H$. 
    The perturbed images are obtained by $\tbsx = (1-\bsalpha)\bsx+\bsalpha\bss$ and the explainer is defined by $\bsalpha(\bsalpha-1) \Delta$. 
    It should be emphasized that, in this derivation, we are not considering the random shifts and spatial dependencies in the derivation of RISE for simplicity.
    Although the use of finite difference might not be obvious in the original papers, one can show that the prediction-based explainer is essentially a scaled finite difference up to a constant; see \cref{appendix:finite_difference} for details. 
    The perturbation distribution for this method can be identified by a Dirac Comb $p(\tbsx)=\sum_j \delta(\tbsx_j-\bsx_j)+\delta(\tbsx_j-\bss_j)$, where $j$ is on the superpixels \cite{rahman_applications_2011}. 
    The PSD of this kernel is of the form
    $\hp(\bsomega) = \sum_j e^{2i\pi\bsx_j}+e^{2i\pi\bss_j}$.
    Closely related to the Dirac Comb distribution is when, instead of masking pixels, a Gaussian filter is applied to the image \cite{fong_interpretable_2017}, unfortunately, the hand-designed processes further complicate the theoretical study of their results. 
    
    \item \textbf{Rect Kernel:}
    In \textbf{IntegratedGrad} \cite{sundararajan_axiomatic_2017}, the explainer is set to $(\bsx-\bss)\nabla$, in which $\bss$ is a fixed baseline, so $(\bsx-\bss)$ acts like a constant scaling in time domain. The perturbed images are set to $\tbsx = (1-\sigma)\bsx+\sigma\bss$ where $\sigma\sim\operatorname{Unif}(0,1)$. Moreover, with a change of variables, one can show that the perturbation distribution is set to $p(\tbsx)= \frac{1}{\sigma|\bss-\bsx|}\operatorname{Rect}(\frac{\tbsx-\bsx}{\sigma(\bss-\bsx)}-\frac{1}{2})$. 
    Intuitively, the sampling distribution forms a rectangular pulse from the baseline to the image.
    The chosen baseline affects the timescale and translation of the rectangular pulse.
    We analyze this kernel with more detail in \cref{rect-kernel-analysis}.
    
    \item \textbf{Gaussian Kernel:}
    Smilkov et al. \cite{smilkov_smoothgrad_2017} propose \textbf{SmoothGrad}, where they set explainer to $\nabla$ and distribution to $p(\tbsx) = \mathcal{N}(\bsx,\sigma^2)$ and suggest a heuristic for finding an optimal value for $\sigma$.
    We show the spectral properties of this distribution in \cref{sect:guassian-kernel}. 

\end{itemize}

%
\section{Finite Difference Explainers}
\label{appendix:finite_difference}
To define an explainer, some of the explanation methods use the prediction of the model on a given sample and its perturbations \cite{petsiuk_rise_2018}. We show that their explainer can be written in the form of a finite difference up to a constant with respect to the perturbation distribution. Assume $\bsalpha$ be the multiplicative perturbation mask sampled from some perturbation distribution with expected value of $\bs{M}$. Then we have
\begin{align}
    \operatorname{Prediction Explainer}(f,\bsx) &= \E_{p(\tbsx)}[\bsalpha f(\tbsx)]\\ 
    &= \E_{p(\tbsx)}[\bsalpha f(\tbsx)-\bsalpha f(\bsx)+\bsalpha f(\bsx)]\\
    &= \E_{p(\tbsx)}[\bsalpha(\tbsx-\bsx) \left(\frac{f(\tbsx)- f(\bsx)}{\tbsx-\bsx}\right)+\bsalpha f(\bsx)]\\
    &= \E_{p(\tbsx)}[\bsalpha(\tbsx-\bsx) \Delta f(\tbsx)+\bsalpha f(\bsx)]\\
    &= \E_{p(\tbsx)}[\bsalpha(\tbsx-\bsx) \Delta f(\tbsx)]+f(\bsx)\E_{p(\tbsx)}[\bsalpha]\\
    &= \E_{p(\tbsx)}[\bsalpha(\tbsx-\bsx) \Delta f(\tbsx)]+\bs{M}f(\bsx)
\end{align}.
Note that the subtraction does not change the visualizations due to the normalization.
Hence, after normalization, one can write $\operatorname{Prediction Explainer}(f,\bsx)=\E_{p(\tbsx)}[\bsalpha (\tbsx-\bsx) \Delta f(\tbsx)]$, which means that the explainer is set to $\bsalpha (\tbsx-\bsx) \Delta$.
This demonstrates that these methods are also estimating the gradient, and the low-pass filter is typically attained by having a very low resolution for the perturbation mask.

Another way to understand the low-pass filter is comparing the derived finite difference by the definition of gradient. In the implementations of finite difference methods, we use expectation over perturbed samples, which is equivalent to assuming a constant denominator in the definition of finite difference. This means that small perturbations to the input are diminished (compared to gradient, where the denominator approaches to infinity for small perturbations).
\section{Notes on Reproducibility}
\label{reproducibility}
It is essential to recognize that Explainable Artificial Intelligence (XAI) should be integrated into the broader AI development pipeline. Consequently, the design of XAI methods must account for several logical considerations inherent to AI pipeline construction. A critical yet often overlooked aspect is the assumption that the label is unknown during the inference. This assumption should be consistently maintained when generating explanations. As a result, methods that rely on the label for explanation computation are likely to be less effective in practical applications. That being said, in all of our implementations, we have used the most probable class in the prediction as the class that needs to be explained. This, in turn, can make explanations less useful if the model has a low accuracy. Furthermore, we considered using log-softmax of the model output, as suggested by the recent work \cite{wang_why_2022}.

\subsection{Choosing a Proper Perturbation Cutoff}
An underexplored subject regarding the practical details of generating explanations pertains to selecting a proper cutoff for perturbation scale. This is crucial for ensuring both reproducibility and a methodical approach to implementing explanation methods. Our experimental findings suggest that monitoring the trends in predictive entropy of a classifier across varying noise scales can serve as a heuristic. We have found that when working with Gaussian noise, we typically observe an anomalous trend in the predictive entropy of the classifier, as illustrated in \cref{fig:predictive-entropy}. 
It is worth noting that the values presented for perturbation levels in the main text are relative to this maximum cutoff \ie $\sigma=1.0$, refers to one unit of maximum perturbation level.

We attribute this anomalous trend in predictive entropy to two primary factors. Firstly, the noise level scales beyond the range that is present in the training data, rendering it out-of-distribution. Secondly, the scale of logits is affected by the scale of the input. In probability space, \ie values after the softmax, the large scale of logits results in an overconfident prediction.

\begin{figure*}[h]
    \centering
    \includegraphics[width=\textwidth]{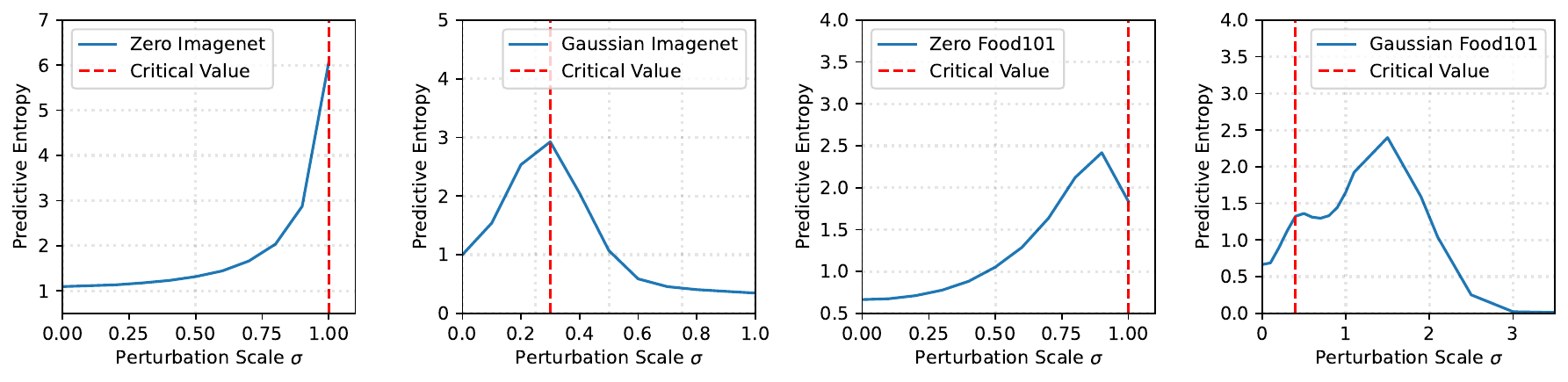}
    \caption{
    This figure shows the ResNet50's predictive entropy at varying perturbation scales on two datasets. As can be seen in the plots, when the perturbation distribution is set to Gaussian, it increases gradually and reaches to a point where an anomalous trend occurs. The anomaly shows itself as a decrease in the predictive entropy in ImageNet and plateauing in Food101, despite our expectation that it should increase to its highest theoretical value. Since after the anomalous trend the visual quality of explanations degrade, we have selected the noise scale cutoff right before the anomaly starts, shown with red dashed lines.
    }
    \label{fig:predictive-entropy}
\end{figure*}

\subsection{Controlling Signal-to-Noise Ratio}
\label{sec:signal-to-noise-ratio}
In the context of our probabilistic framework, as detailed in \cref{appendix:unifying}, various methods employ different noise-signal combination strategies to obtain perturbed input images, visualized in \cref{fig:combination-fn}. In our implementation for reproducing SmoothGrad, we adopt an additive combination method. However, for SpectralLens and measuring the cosine similarity (as discussed in \cref{aggregation-of-information,kc-similarity-subsection}), we employ a convex combination strategy.
Notably, experimental findings indicate that the model's sensitivity lies more in the signal-to-noise ratio rather than the specific combination function employed. Nonetheless, using an additive combination method may result in input scaling and is unsuitable for identifying an optimal perturbation scale in \cref{kc-similarity-subsection}.

\subsection{Implementation Details} 
We have used ResNet50\cite{he_deep_2015} and ViT\cite{dosovitskiy_image_2021} on ImageNet and Food101 to generate explanations for corrupted images with an explanation prior $p(\sigma)~=~\operatorname{Unif}(0,1)$.
(see \cref{tab:vit} for ViT-B/16\footnote{Model architecture with pretrained weights are used from \url{https://github.com/google-research/vision_transformer}} architecture). 
The sampling pipeline was implemented in JAX for improved runtime efficiency\footnote{Our JAX implementation is approximately 10 times faster than its PyTorch counterpart, albeit with increased pipeline complexity, available at \url{https://github.com/Amir-Mehrpanah/an_explanation_model/tree/berzelius}. The complexity is due to the end-to-end compilation of the sampling; however, the results are reproducible with less complexity in PyTorch.}.
We used Monte Carlo estimation of expectations while sampling as many samples as the change in the estimates is higher than $5e-3$, leading to $256$ samples per image on average.
We used maximum of $8$K samples for RISE and Occlusion, which led to a longer runtime.
For determining the proper Gaussian noise scale we performed grid search on the scales while monitoring the predictive entropy, and we selected the perturbation scale at which the entropy plateaus see \cref{fig:predictive-entropy} for more details.
Example results are presented in \cref{fig:frequency-explanations} (additional examples are given in \cref{appendix:example-outputs}). We leave the analysis of existing evaluation methods from a spectral point of view for future work.

\section{An Ablation Study of Architecture}
\label{architectures-appendix}
Although we used ResNet50 due to its prevalence in the main text, it should be mentioned that we are agnostic of architecture in the theoretical analysis. To verify this claim, we have provided an evaluation of the stated design decisions on \texttt{ViT-B/16} in \cref{tab:vit}.  
\begin{table}
\caption{
The table presents an evaluation of the design choices, with \texttt{ViT-B/16} on 10K images of ImageNet validation set. As can be seen, the trend is very similar to that of the \cref{tab:evaluations}.
}
\label{tab:vit}
\centering
\begin{tabular}{llcc}
\toprule
& & Insertion($\uparrow$) & Deletion($\downarrow$)\\ 
\cmidrule(lr){3-4}
\parbox[t]{3mm}{\multirow{5}{*}{\rotatebox[origin=c]{90}{Gradients}}} 
& OS & 
$0.384\pm0.017$ & $0.331\pm0.012$ 
\\
& RISE & 
$\overline{0.412}\pm0.016$ & $0.298\pm0.014$
\\ 
& SG &
$0.230\pm0.010$ & $0.213\pm0.010$ 
\\
& XIG &
$0.227\pm0.009$ & $0.239\pm0.011$
\\ 
& IG & 
$0.227\pm0.009$ & $0.210\pm0.011$  
\\  
\midrule
\parbox[t]{5mm}{\multirow{5}{*}{\rotatebox[origin=c]{90}{Gradients$^2$}}}
& SG$^2$ & 
$\mathbf{0.375}\pm0.017$ & $0.209\pm0.011$ 
\\
& XIG$^2$ &
$\mathbf{0.396}\pm0.018$ & $\mathbf{0.211}\pm0.009$ 
\\ 
& IG$^2$ & 
$\mathbf{0.401}\pm0.018$ & $\mathbf{0.186}\pm0.010$ 
\\ 
\cmidrule{2-4}
& $\text{SG}^2_{\text{Opt.}}$ &
$\overline{\mathbf{0.412}}\pm0.018$ & $\mathbf{0.170}\pm0.008$ 
\\
& SL$^2$ & 
$\overline{\mathbf{0.411}}\pm0.020$ & $\underline{\mathbf{0.149}}\pm0.008$ 
\\
\bottomrule
\end{tabular}
\end{table}

\begin{figure}
    \centering
    \includegraphics[width=0.4\textwidth,angle=90]{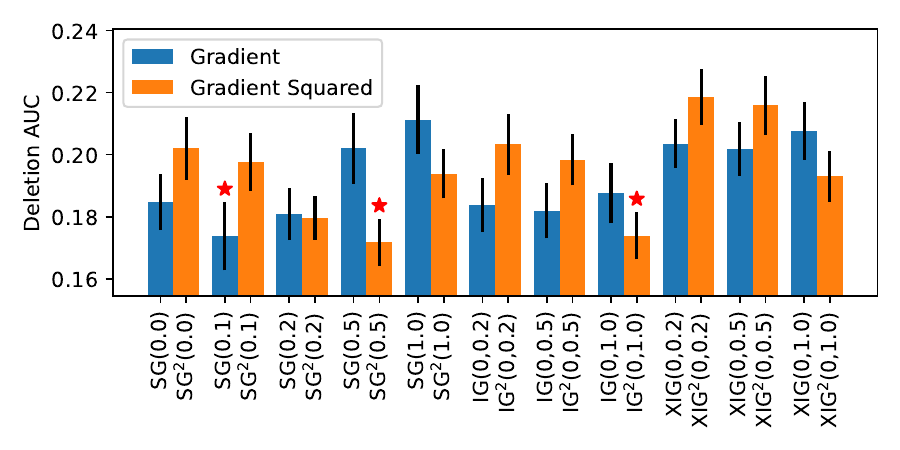}
    \includegraphics[width=0.4\textwidth,angle=90]{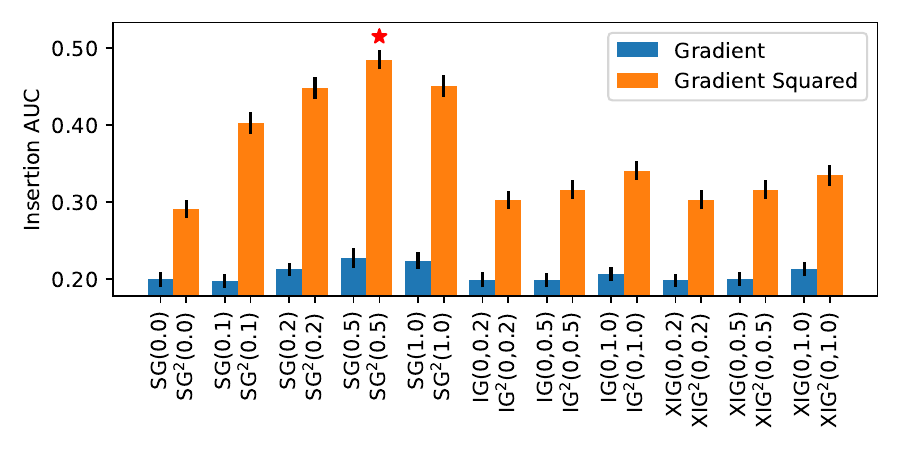}
    \caption{
    This figure illustrates the deletion and insertion scores produced by various explanation techniques, such as SG, IG, XIG. 
    The evaluation was carried out using ResNet50 on a subset of the ImageNet validation set, consisting of 10 samples of 1k size each (total of 10k images).
    Note that, unlike IG, XIG utilizes multiplication of the input image for explanation. 
    Furthermore, SG(0.0) corresponds to the VanillaGradient method. 
    Higher insertion scores suggest better performance, while lower is better in deletion scores. 
    All kernels in this figure employ a Dirac explanation prior, shown in parentheses. 
    We have measured the scores for IG and SG with varying values of perturbation levels. 
    The red star indicates an insignificant p-value obtained from a t-test comparing the most effective method.
    As can be seen from the results, squaring gradients consistently leads to better performance, with only a few exceptions in the deletion scores.
    It is important to note that the most effective method among those employing a Dirac explanation prior can be identified based on the highest similarity (see \cref{fig:kc-similarity-dataset}). 
    }
    \label{fig:imagenet-deletion-insertion-scores}
\end{figure}


\begin{figure*}
    \centering
    \includegraphics[width=0.6\textwidth]{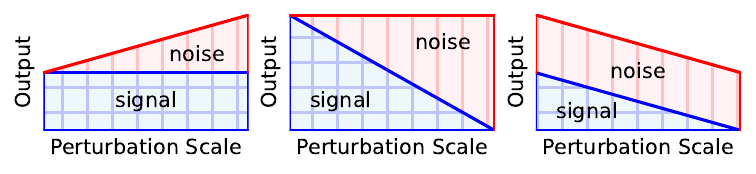}
    \caption{
    This figure illustrates a visual representation of different combination functions. From left to right we have, additive combination, convex combination, and (signal) damping combination. It is noteworthy that different methods may employ distinct combination functions. For example, RISE and occlusion techniques utilize a convex combination function, IntegratedGrad applies a signal damping combination, and SmoothGrad employs an additive combination approach.
    }
    \label{fig:combination-fn}
\end{figure*}

\section{An Ablation Study of Pixel-based Evaluation}
\label{sec:ablation-erosion}
As the Fourier transform in our work, is \textit{per pixel} (1D), and \textit{not} along spatial dims (2D), which is more common, it expands on input dimensions, due to the linearity of the FT operator. The function's FT \ie $\widehat{f}$, is the connecting piece of explanation pixels.
Therefore, while large/small structures are spatially low/high-frequency, it is \textit{not} necessarily the case in our analysis. 

The purpose of \cref{tab:evaluations} presenting AUC of single-pixel \textit{MoRF} is to empirically validate \textit{our} theoretical findings, which are on pixel values, and not to rank methods for their efficacy. That is why we believe single-pixel evaluation is the right choice to provide empirical evidence in our case. Importantly, as stated in before, the frequency bias is not only in explanations, but also in evaluation. Hence, interestingly, the effect of using blobs for evaluation (instead of pixels) is \textit{another} low-pass filter on gradients, along spatial dimensions, removing frequencies that \textit{could have been removed} in explanation step by a \textit{different} perturbation kernel.

Yet, due to general concerns about pixel evaluation of explanations, we included this section to perform an ablation study on the kernel size of an erosion, which is performed after selecting the top X\% of pixels in an explanation, presented in \cref{tab:ablation-erosion}.
It is noteworthy that because of erosion kernel, the actual removed area is usually higher than the dedicated budget which is visualized in \cref{fig:ablation_kernel,fig:ablation_kernel_area}. Therefore, for a fair comparison, we use the actual removed area (instead of the dedicated budget) when computing MoRF AUC.

\begin{table}
\caption{
This table presents an ablation study of the kernel size of erosion kernel that is used for removing patches instead of single pixels when computing the MoRF AUC. For this table, we used ResNet50 on 10K images from ImageNet validation set. 
As can be seen, the trends justified by our theory still hold with different kernel sizes. Although, the theory in this work, are based on pixels due to linearity of the Fourier operator.
The methods with explicit spatial constraints exhibit minimal changes after the erosion kernel (see \cref{fig:ablation_kernel_area,fig:ablation_kernel} for details).
}
\label{tab:ablation-erosion}
\centering
\resizebox{\textwidth}{!}{%
\begin{tabular}{llcc|cc|cc|cc} %
\toprule
&& \multicolumn{2}{c}{($1\times1$) Erosion} & \multicolumn{2}{c}{($3\times3$) Erosion} 
& \multicolumn{2}{c}{($5\times5$) Erosion}
& \multicolumn{2}{c}{($7\times7$) Erosion}
\\ 
\cmidrule(lr){3-4}
\cmidrule(lr){5-6}
\cmidrule(lr){7-8}
\cmidrule(lr){9-10}
& & Insertion($\uparrow$) & Deletion($\downarrow$) & Insertion($\uparrow$) & Deletion($\downarrow$) 
& Insertion($\uparrow$) & Deletion($\downarrow$) 
& Insertion($\uparrow$) & Deletion($\downarrow$) 
\\ 
\cmidrule(lr){3-10}
\parbox[t]{5mm}{\multirow{5}{*}{\rotatebox[origin=c]{90}{Gradients}}} 
& OS & 
  $0.384\pm0.008$ & $0.237\pm0.010$ 
& $0.385\pm0.008$ & $0.236\pm0.009$ 
& $0.387\pm0.009$ & $0.235\pm0.009$
& $0.389\pm0.008$ & $0.234\pm0.010$
\\
& RISE & 
  $\overline{0.494}\pm0.014$ & $0.174\pm0.007$ 
& $\overline{0.494}\pm0.013$ & $0.174\pm0.006$ 
& $\overline{0.495}\pm0.014$ & $0.175\pm0.006$
& $\overline{0.496}\pm0.014$ & $0.176\pm0.007$
\\ 
& SG &
  $0.212\pm0.010$ & $0.179\pm0.009$ 
& $0.350\pm0.013$ & $\underline{0.137}\pm0.005$ 
& $0.384\pm0.013$ & $\underline{0.149}\pm0.004$
& $0.408\pm0.012$ & $0.175\pm0.004$
\\
& XIG &
  $0.208\pm0.009$ & $0.204\pm0.008$ 
& $0.285\pm0.011$ & $0.204\pm0.007$ 
& $0.322\pm0.013$ & $0.199\pm0.006$ 
& $0.357\pm0.014$ & $0.211\pm0.005$
\\ 
& IG &
  $0.201\pm0.010$ & $0.182\pm0.009$ 
& $0.301\pm0.011$ & $0.177\pm0.007$ 
& $0.328\pm0.012$ & $0.187\pm0.007$
& $0.360\pm0.012$ & $0.212\pm0.007  $
\\  
\midrule
\parbox[t]{5mm}{\multirow{5}{*}{\rotatebox[origin=c]{90}{Gradients$^2$}}}
& SG$^2$ & 
  $\mathbf{0.465}\pm0.013$ & $\underline{\mathbf{0.157}}\pm0.007$ 
& $\mathbf{0.459}\pm0.012$ & $0.157\pm0.007$ 
& $\mathbf{0.455}\pm0.012$ & $0.155\pm0.006$
& $\mathbf{0.458}\pm0.011$ & $\underline{\mathbf{0.157}}\pm0.005$
\\
& XIG$^2$ &
  $\mathbf{0.334}\pm0.014$ & $\mathbf{0.187}\pm0.008$ 
& $\mathbf{0.358}\pm0.012$ & $\mathbf{0.185}\pm0.008$ 
& $\mathbf{0.358}\pm0.012$ & $\mathbf{0.185}\pm0.005$
& $0.373\pm0.012$ & $\mathbf{0.193}\pm0.004$
\\ 
& IG$^2$ & 
  $\mathbf{0.340}\pm0.012$ & $\mathbf{0.168}\pm0.007$  
& $\mathbf{0.358}\pm0.013$ & $0.177\pm0.006$ 
& $\mathbf{0.357}\pm0.013$ & $0.183\pm0.005$
& $0.371\pm0.012$ & $\mathbf{0.195}\pm0.005$
\\ 
\cmidrule{2-10}
& $\text{SG}^2_{\text{Opt.}}$ &
  $\overline{\mathbf{0.491}}\pm0.015$ & $0.187\pm0.007$ 
& $\mathbf{0.473}\pm0.013$ & $0.197\pm0.007$
& $\mathbf{0.472}\pm0.013$ & $0.195\pm0.007$ 
& $\mathbf{0.472}\pm0.012$ & $0.197\pm0.007$
\\
& SL$^2$ & 
  $\overline{\mathbf{0.493}}\pm0.016$ & $0.169\pm0.007$ 
& $\mathbf{0.468}\pm0.013$ & $0.191\pm0.008$ 
& $0.459\pm0.013$ & $0.195\pm0.008$
& $0.459\pm0.013$ & $0.200\pm0.008$
\\
\bottomrule
\end{tabular}
}
\end{table}

\begin{figure*}
    \centering
    \includegraphics[width=0.19\textwidth]{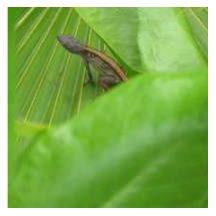}
    \includegraphics[width=0.19\textwidth]{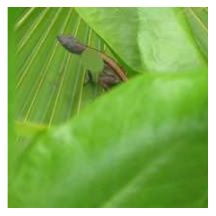}
    \includegraphics[width=0.19\textwidth]{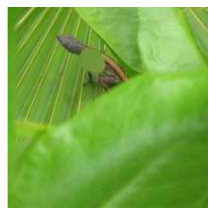}
    \includegraphics[width=0.19\textwidth]{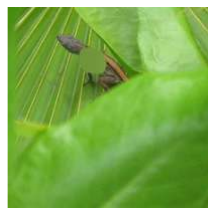}
    \includegraphics[width=0.19\textwidth]{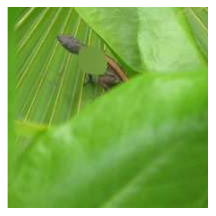}
    \includegraphics[width=0.19\textwidth]{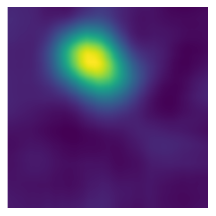}
    \includegraphics[width=0.19\textwidth]{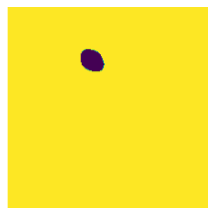}
    \includegraphics[width=0.19\textwidth]{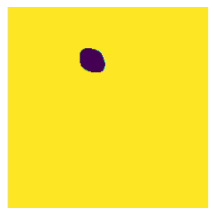}
    \includegraphics[width=0.19\textwidth]{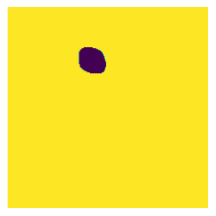}
    \includegraphics[width=0.19\textwidth]{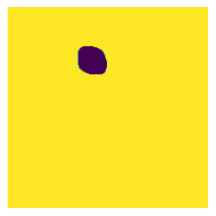}
    \includegraphics[width=0.19\textwidth]{figs/kernel_ablation/Original.pdf}
    \includegraphics[width=0.19\textwidth]{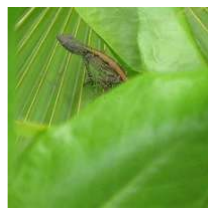}
    \includegraphics[width=0.19\textwidth]{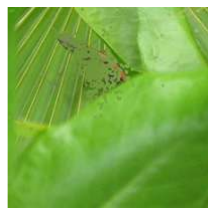}
    \includegraphics[width=0.19\textwidth]{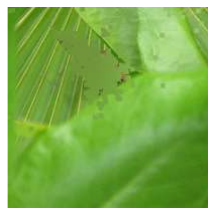}
    \includegraphics[width=0.19\textwidth]{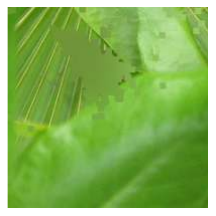}
    \begin{subfigure}[b]{0.19\textwidth}
    \includegraphics[width=\textwidth]{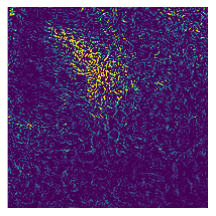}
         \caption{}
    \end{subfigure}
    \begin{subfigure}[b]{0.19\textwidth}
    \includegraphics[width=\textwidth]{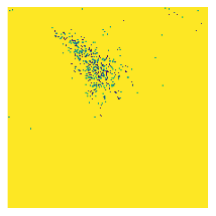}
         \caption{}
    \end{subfigure}
    \begin{subfigure}[b]{0.19\textwidth}
    \includegraphics[width=\textwidth]{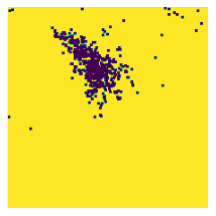}
         \caption{}
    \end{subfigure}
    \begin{subfigure}[b]{0.19\textwidth}
    \includegraphics[width=\textwidth]{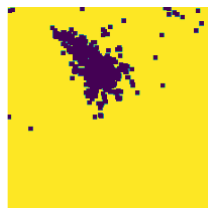}
         \caption{}
    \end{subfigure}
    \begin{subfigure}[b]{0.19\textwidth}
    \includegraphics[width=\textwidth]{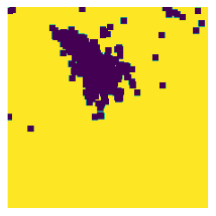}
         \caption{}
    \end{subfigure}
    \caption{This figure visualizes the ablation study on pixel-based evaluation with erosion kernel of varying sizes. On column (a) we have shown the original image and its saliency (RISE top and VG bottom), and on other columns, we have shown the image after removing the patches, corresponding to the removal mask below. The kernel size ranges from (1$\times$1) in (b) (aka pixel-based), to (3$\times$3) in (c), (5$\times$5) in (d), and (7$\times$7) in (e). As the analysis is done based on pixels, we used pixel-based evaluation (b) in the main text. But we can apply erosion to remove or insert patches instead of pixels while computing the deletion or insertion score.
    For a fair comparison, it is important to note that for this visualization, we set the removal area budget to the top 1\% of the saliency. This budget is achieved in pixel-based removal with a 1\% removed area in (b), however due to the erosion, the removed area for columns (c) to (e) increases to 5\%, 8\%, and 11\% for VG, respectively, see \cref{fig:ablation_kernel_area} for more details.}
    \label{fig:ablation_kernel}
\end{figure*}

\begin{figure*}
    \centering
    \includegraphics[width=\textwidth]{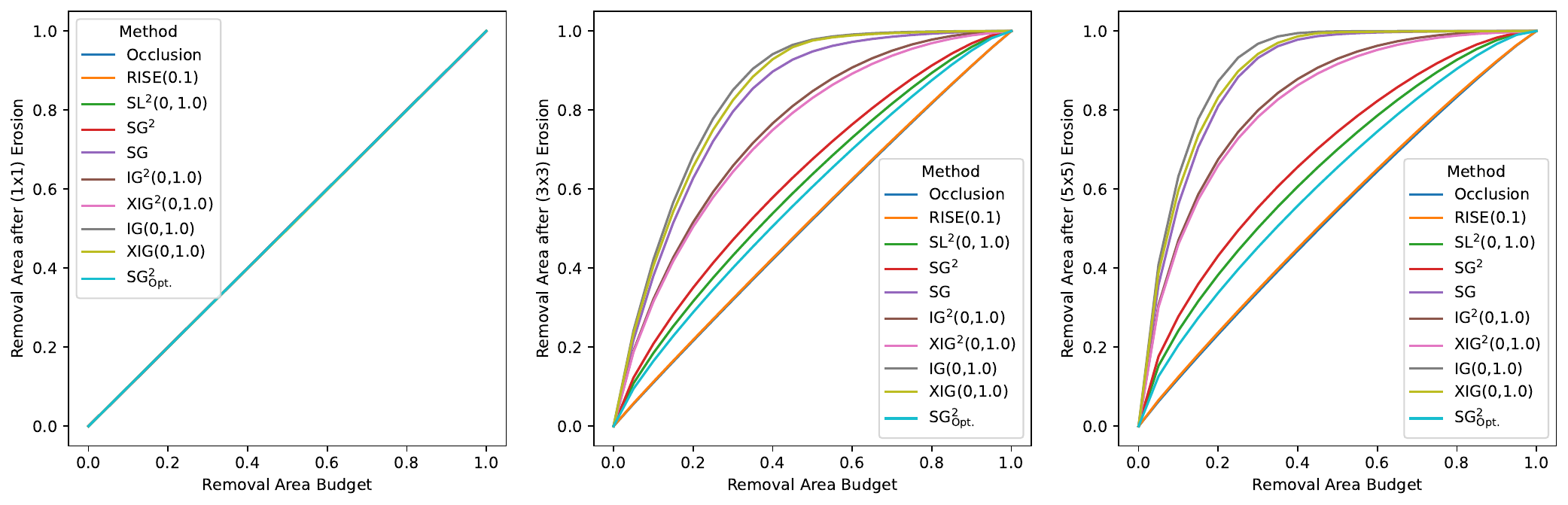}
    
    \caption{This figure visualizes the progression of the removed area on 10K images after erosion, where each curve corresponds to an explanation method, the x-axis shows the budget for removal area of top X\% and the y-axis, shows the actual removed area. As expected, the actual removed area is always more than the budget, we use the actual removed area instead of the decided budget in computing AUCs (justifying the drop after erosion kernel in the performance of sparse explanation methods in \cref{tab:ablation-erosion}). Given that the area under this curve can be a heuristic measure for explanation sparsity, an interesting observation is that unlike the methods with spatial constraints, such as RISE, our approach, SG$_{\text{Opt.}}$ produce spatially smoother explanations, leading to smaller increase in removal area after erosion.}
    \label{fig:ablation_kernel_area}
\end{figure*}

\section{Quantification of Explanation Inconsistency}
As we have mentioned the inconsistency of explanations across hyperparameters in the paper, it is useful to quantify this term, so that we gain deeper insight and know potential underlying patterns.
It should be noted that the inconsistency in each explanation can be used as a heuristic for unreliability or uncertainty in explanations.

We define the inconsistencies in explanations quantitatively as follows,
\begin{align}
    \text{Inconsistency}(x) 
    &= 1- \E_{p(\sigma)}[\text{Cosine-similarity}(e_{0},e_{\sigma})] \\
    &= 1 - \int \text{Cosine-similarity}(e_{0},e_{\sigma})p(\sigma)d\sigma \\
    &\approx 1 - \frac{1}{N}\sum_i^N \text{Cosine-similarity}(e_{0},e_{\sigma_i})
\end{align}
where $e_{\sigma} = \E_{p(\tbsx;\sigma)}[(\nabla f)^2]$ denotes the explanation at hyperparameter $\sigma$, and $p(\sigma) = \text{Unif}(0,1)$. We normalize, downsample and flatten explanations to compute the cosine similarity.
The quantified inconsistency reveals interesting patterns for two explanation methods, yet this quantity is not correlated with predictive entropy. The value of inconsistency is computed per image (each point in the scatter plot visualized in \cref{fig:inconsistency}).

\begin{figure*}
    \centering
    \includegraphics[width=\textwidth]{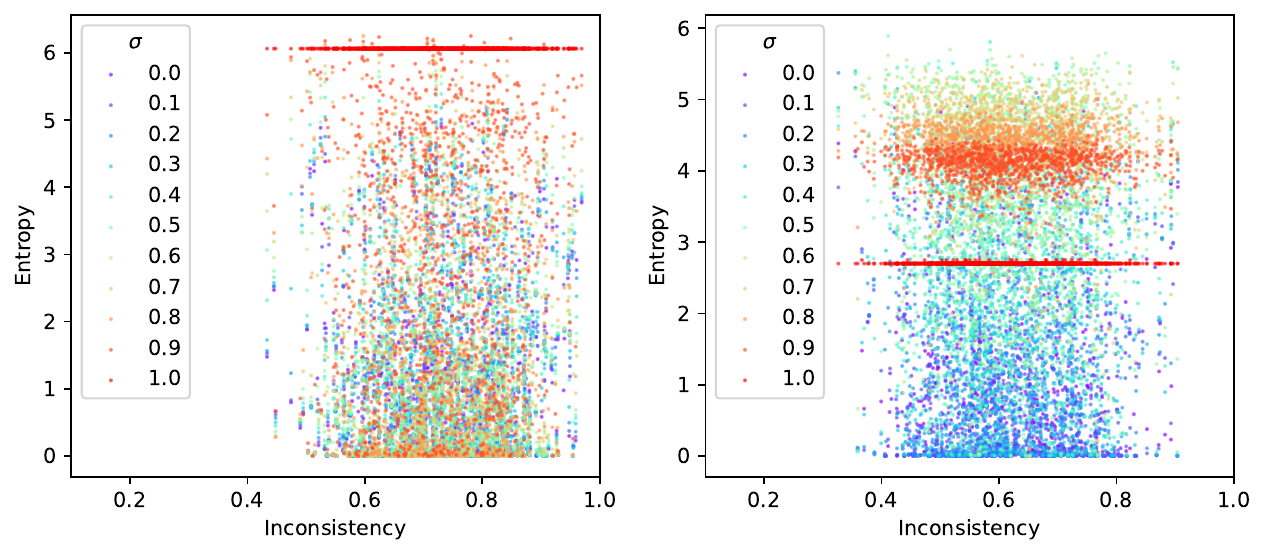}
    \caption{This figure visualizes the inconsistency versus predictive entropy for 1K images from ImageNet at different explanation hyperparameters, for $\text{IG}^2$ on the left, and $\text{SG}^2$ on the right. 
    As expected, the perturbation scale is highly correlated with entropy in $\text{SG}^2$, nonetheless, the pattern is not apparent in $\text{IG}^2$. In both cases, however, it can be seen that there is no clear correlation between predictive entropy and inconsistency.}
    \label{fig:inconsistency}
\end{figure*}

\section{Example Explanations and Their SpectralLens}
\label{appendix:example-outputs}
\subsubsection{A Note on the Locality of SpectralLens.} Although we can break down the input signal to a range of frequencies, it should be noted that the obtained frequencies can only be interpreted locally on an input image and cannot necessarily be generalized to a dataset-level information. 
For example, an image of the sea or the clear sky generates low-frequency explanations in many levels of noise, but it is not necessarily comparable to the low-frequency explanation generated in the background of a car at high noise scales.

\begin{figure*}
    \centering
    \includegraphics[width=0.24\textwidth]{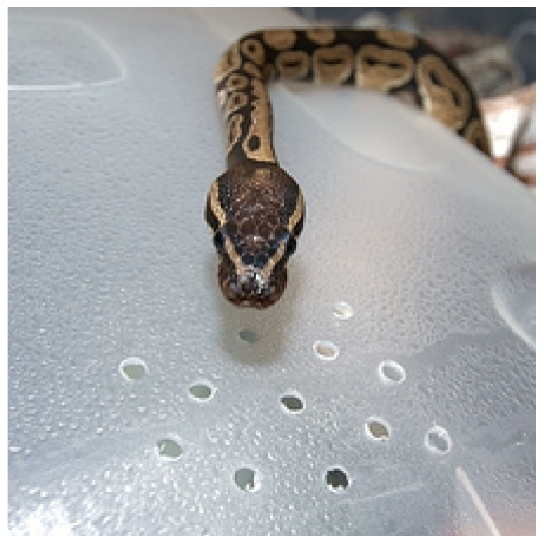}
    \includegraphics[width=0.24\textwidth]{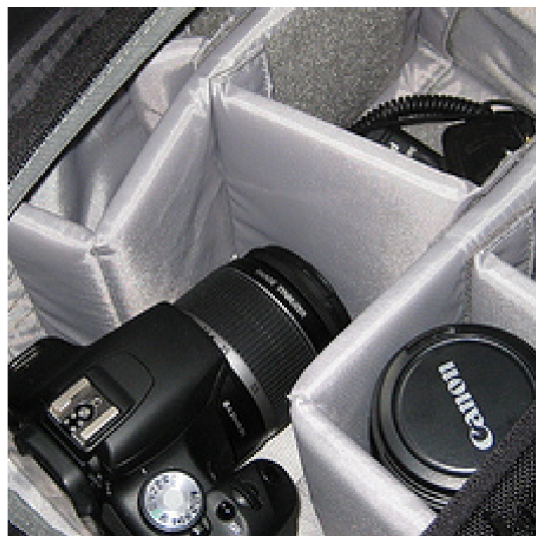}
    \includegraphics[width=0.24\textwidth]{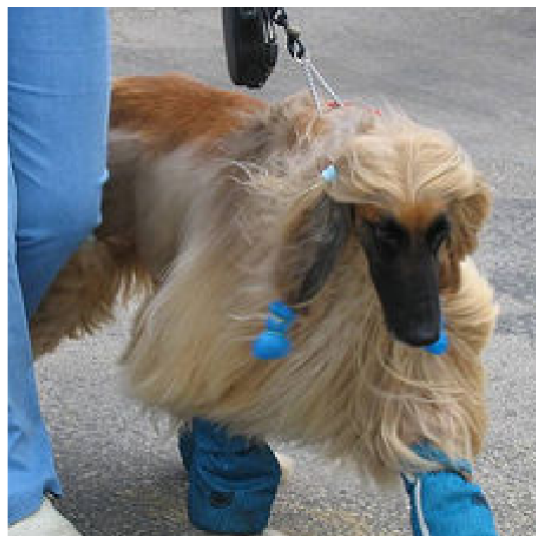}
    \includegraphics[width=0.24\textwidth]{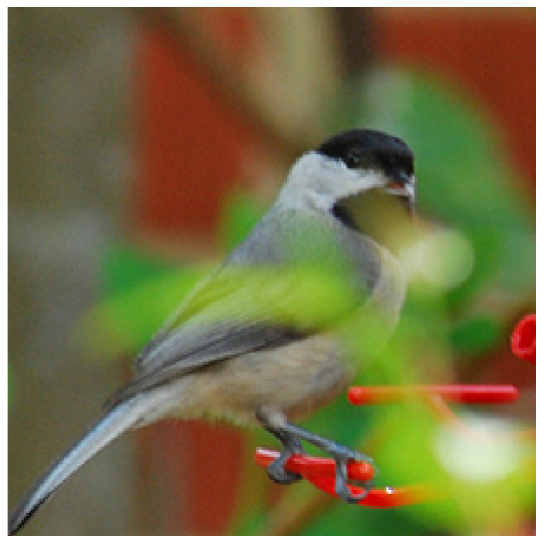}
    \begin{subfigure}[b]{0.24\textwidth}
    \includegraphics[width=\textwidth]{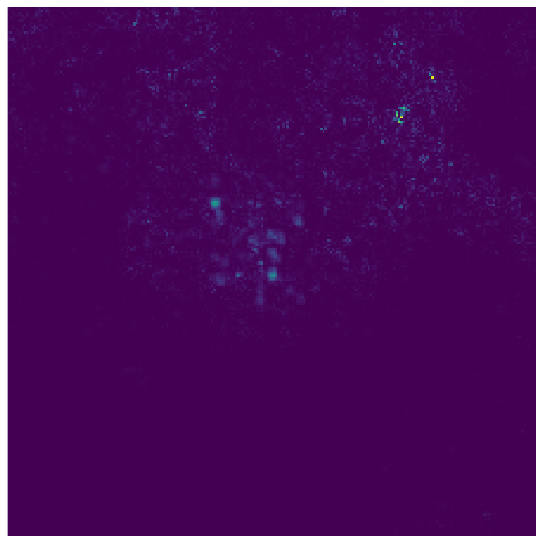}
    \caption{}
    \end{subfigure}
    \begin{subfigure}[b]{0.24\textwidth}
    \includegraphics[width=\textwidth]{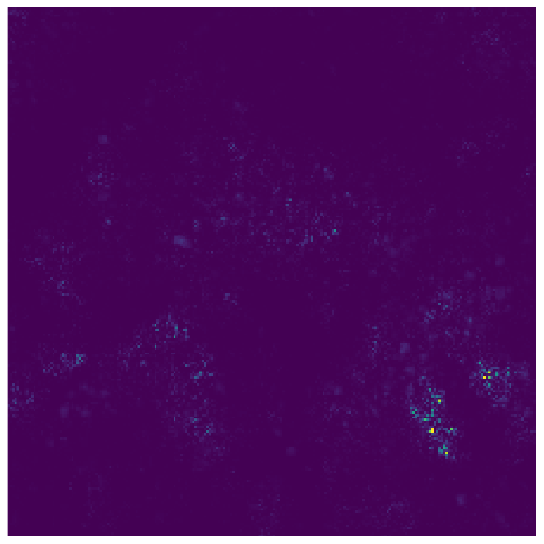}
    \caption{}
    \end{subfigure}
    \begin{subfigure}[b]{0.24\textwidth}
    \includegraphics[width=\textwidth]{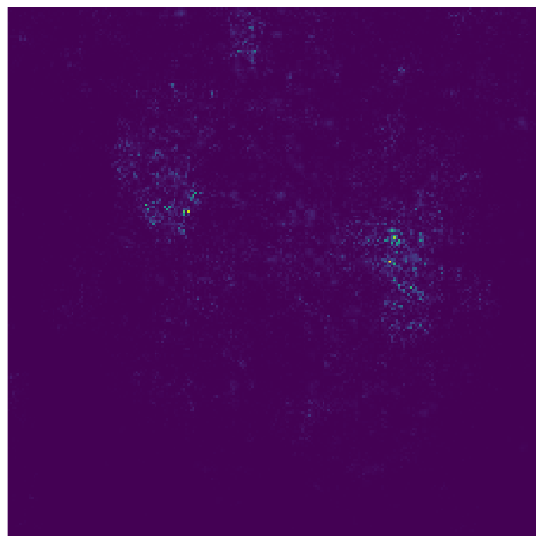}
    \caption{}
    \end{subfigure}
    \begin{subfigure}[b]{0.24\textwidth}
    \includegraphics[width=\textwidth]{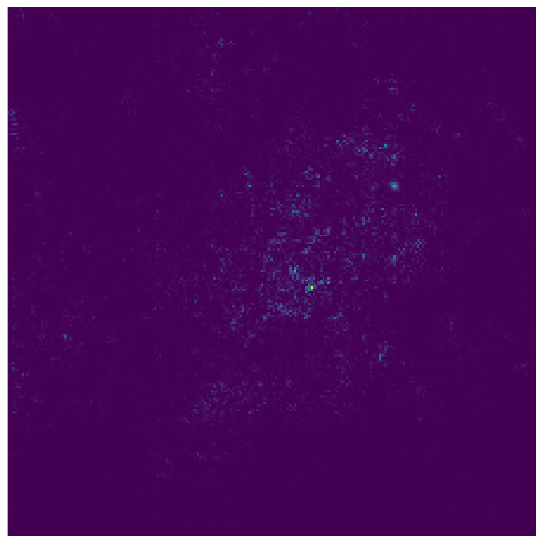}
    \caption{}
    \end{subfigure}
    \caption{This figure visualizes images and their corresponding VanillaGradient$^2$ in each column for ResNet50 trained ImageNet. As can be seen in the visualizations and has been shown before experimentally, the visualizations look sparse and noisy. In our framework, we attribute this phenomenon to the spectral bias of gradient operation which is used in many gradient based explanation methods, see \cref{sect:vanilla-grad} for a more detailed discussion.}
    \label{fig:example-vgsq-row}
\end{figure*}

\begin{figure*}[h]
    \centering
    \includegraphics[width=0.24\textwidth]{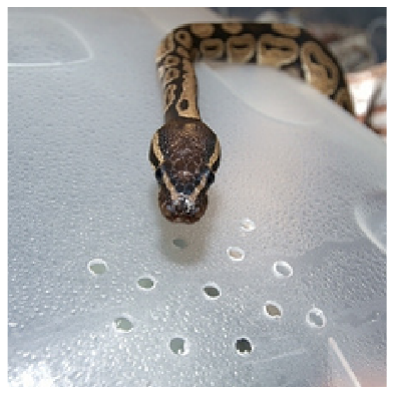}
    \includegraphics[width=0.24\textwidth]{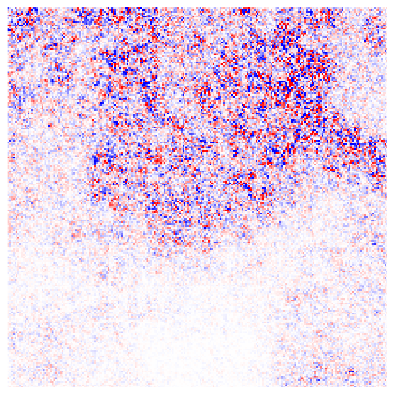}
    \includegraphics[width=0.24\textwidth]{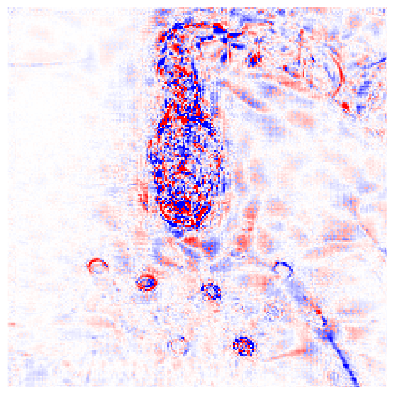}
    \includegraphics[width=0.24\textwidth]{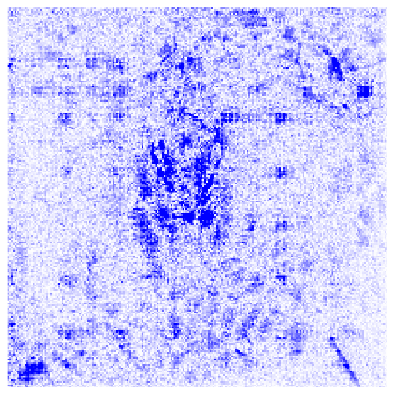}
    \begin{subfigure}[b]{0.24\textwidth}
    \includegraphics[width=\textwidth]{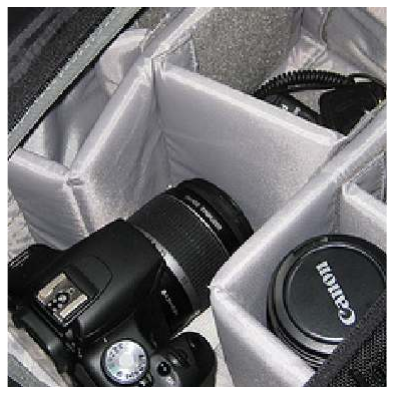}
    \caption{}
    \end{subfigure}
    \begin{subfigure}[b]{0.24\textwidth}
    \includegraphics[width=\textwidth]{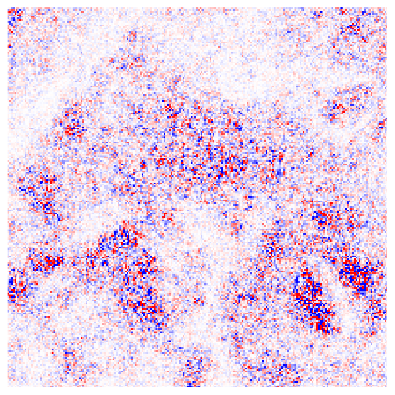}
    \caption{}
    \end{subfigure}
    \begin{subfigure}[b]{0.24\textwidth}
    \includegraphics[width=\textwidth]{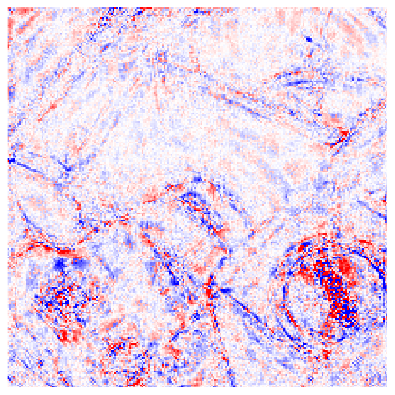}
    \caption{}
    \end{subfigure}
    \begin{subfigure}[b]{0.24\textwidth}
    \includegraphics[width=\textwidth]{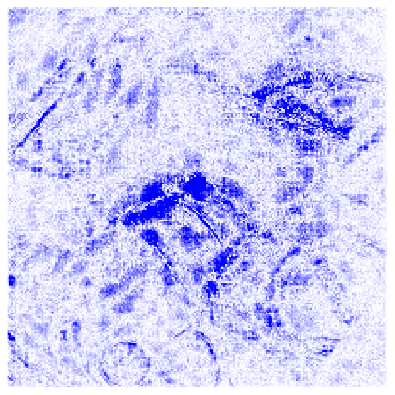}
    \caption{}
    \end{subfigure}
    \caption{In this figure, we have shown SmoothGrad (SG) where each column corresponds to a noise level, namely (b) SG(0.0), (c) SG(0.2) and (d) SG(0.5), for ResNet50 trained on ImageNet. Red pixels indicate the change in the gradient sign relative to (d) SG(0.5). As can be observed here, certain pixels exhibit sign alterations under different hyperparameters. The rationale behind this phenomenon has been unknown but has been justified through our theoretical framework in \cref{rect-kernel-analysis}. The Gaussian kernel and gradient create a band-pass filter in spectral domain, and we can confirm visually that the frequencies in input domain are correlated with frequencies in spatial domain.}
    \label{fig:example-sg-row}
\end{figure*}

\begin{figure*}[h]
    \centering
    \includegraphics[width=0.24\textwidth]{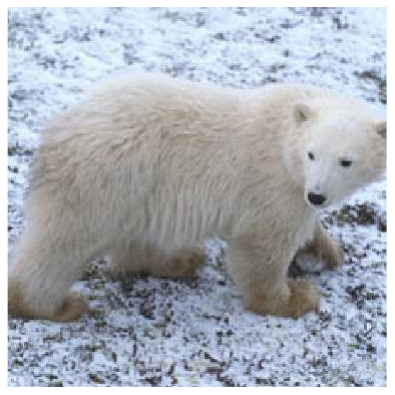}
    \includegraphics[width=0.24\textwidth]{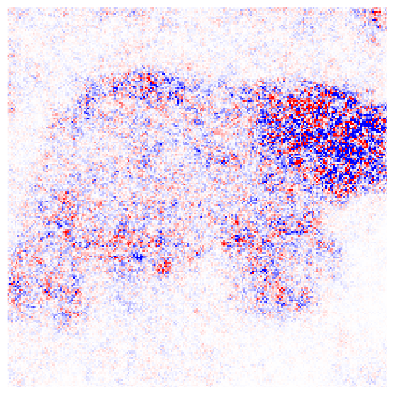}
    \includegraphics[width=0.24\textwidth]{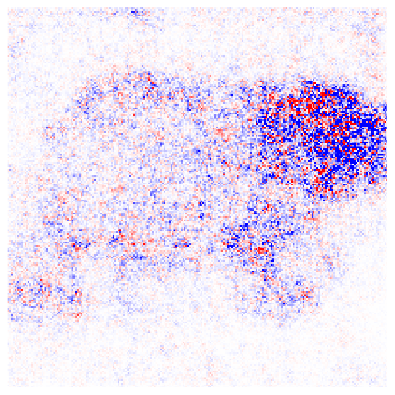}
    \includegraphics[width=0.24\textwidth]{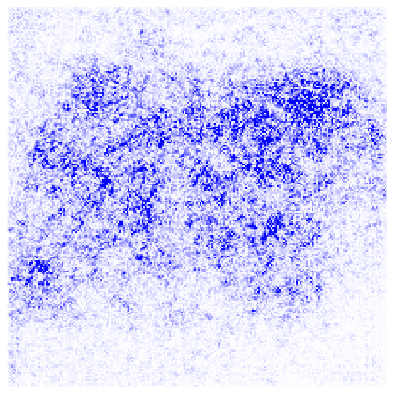}
    \begin{subfigure}[b]{0.24\textwidth}
    \includegraphics[width=\textwidth]{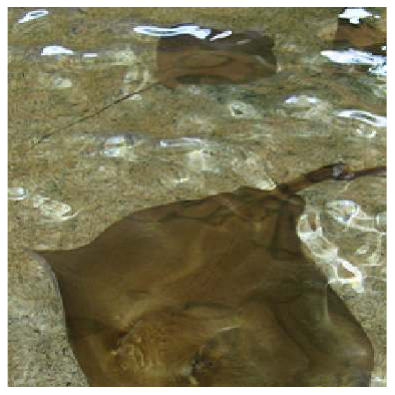}
    \caption{}
    \end{subfigure}
    \begin{subfigure}[b]{0.24\textwidth}
    \includegraphics[width=\textwidth]{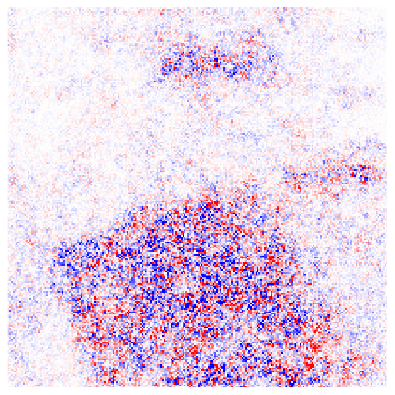}
    \caption{}
    \end{subfigure}
    \begin{subfigure}[b]{0.24\textwidth}
    \includegraphics[width=\textwidth]{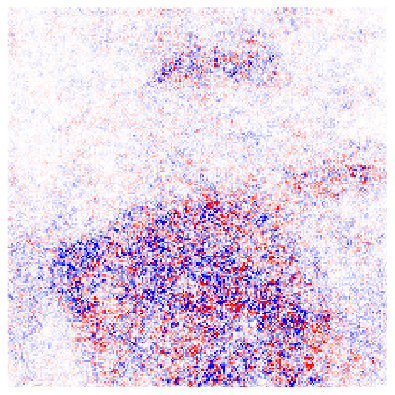}
    \caption{}
    \end{subfigure}
    \begin{subfigure}[b]{0.24\textwidth}
    \includegraphics[width=\textwidth]{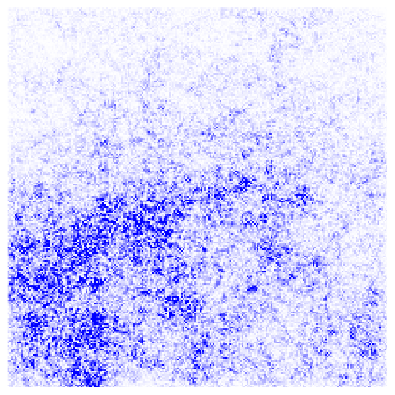}
    \caption{}
    \end{subfigure}
    \caption{This figure illustrates IntegratedGrad (IG) without input multiplication at various hyperparameters for ResNet50 trained on ImageNet. The parameter $\sigma$, governs the integration of gradients between the zero baseline and the input, where a value of $\sigma=0.5$ indicates integration from the image towards the midpoint between the image and the zero baseline. The red pixels denote the alterations in the sign of (b) IG(0.3), and (c) IG(0.6), relative to (d) IG(1.0).}
    \label{fig:example-ig-row}
\end{figure*}

\begin{figure*}[h]
    \centering
    \includegraphics[width=0.24\textwidth]{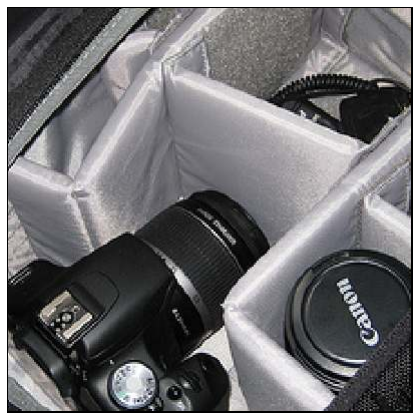}
    \includegraphics[width=0.24\textwidth]{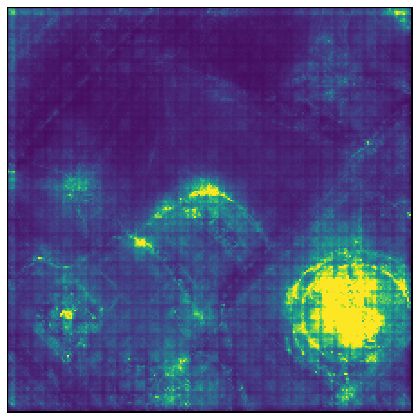}
    \includegraphics[width=0.24\textwidth]{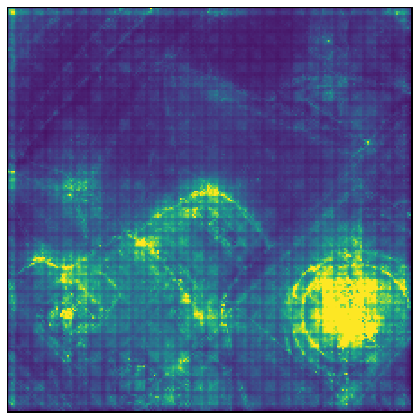}
    \includegraphics[width=0.24\textwidth]{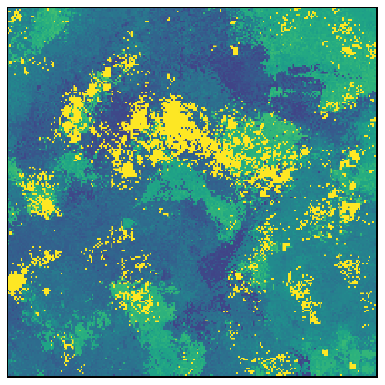}
    \includegraphics[width=0.24\textwidth]{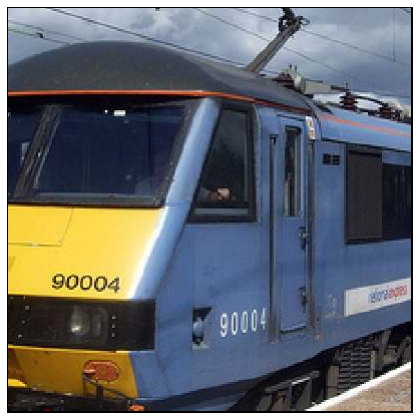}
    \includegraphics[width=0.24\textwidth]{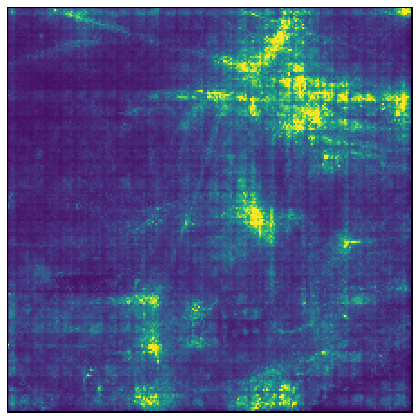}
    \includegraphics[width=0.24\textwidth]{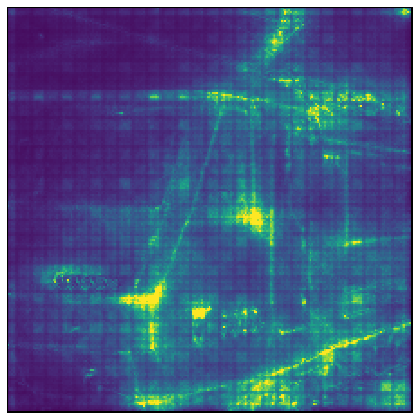}
    \includegraphics[width=0.24\textwidth]{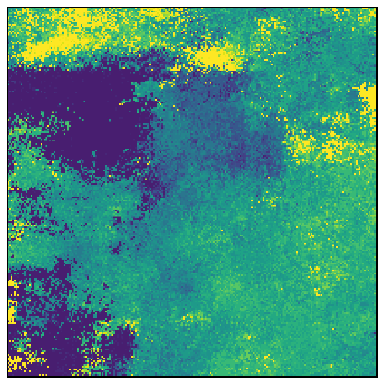}
    \includegraphics[width=0.24\textwidth]{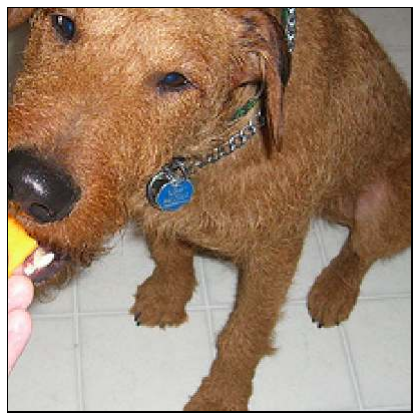}
    \includegraphics[width=0.24\textwidth]{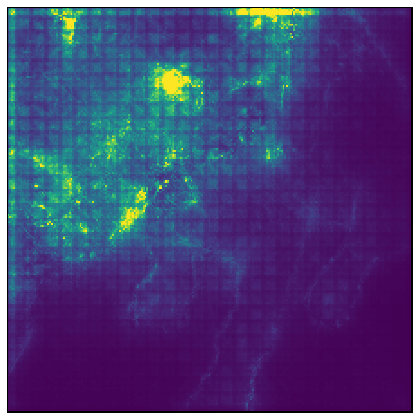}
    \includegraphics[width=0.24\textwidth]{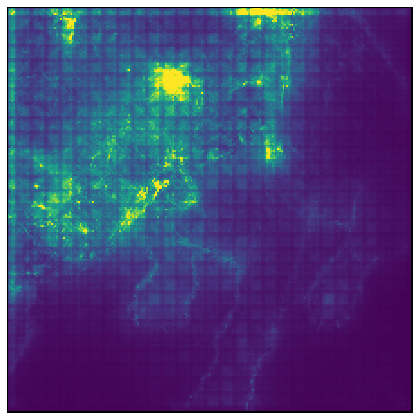}
    \includegraphics[width=0.24\textwidth]{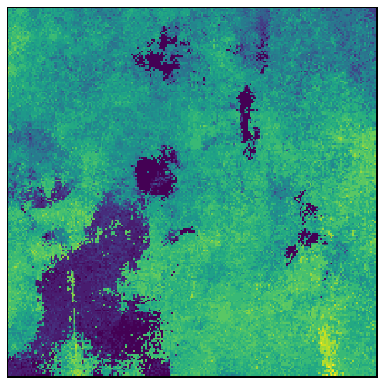}
    \includegraphics[width=0.24\textwidth]{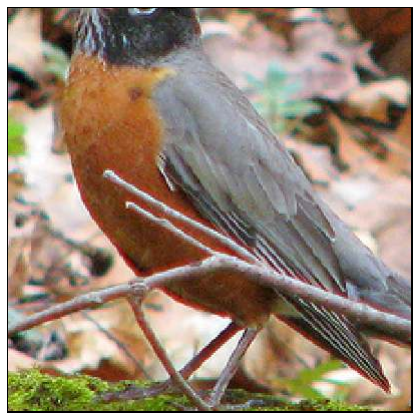}
    \includegraphics[width=0.24\textwidth]{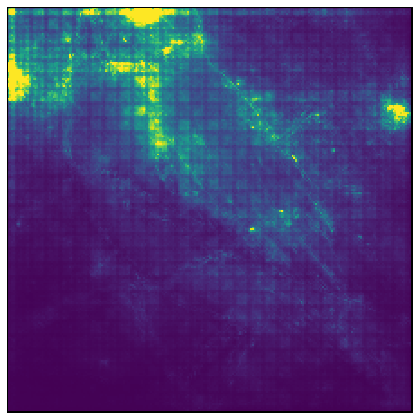}
    \includegraphics[width=0.24\textwidth]{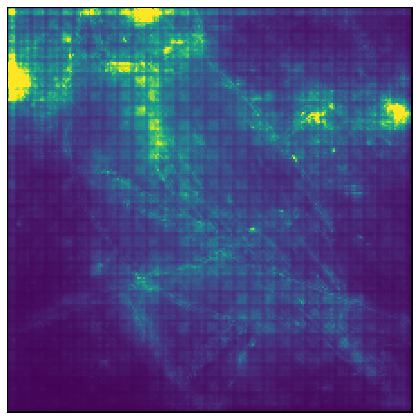}
    \includegraphics[width=0.24\textwidth]{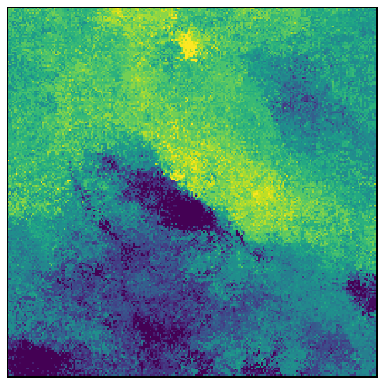}
    \begin{subfigure}[b]{0.24\textwidth}
    \includegraphics[width=\textwidth]{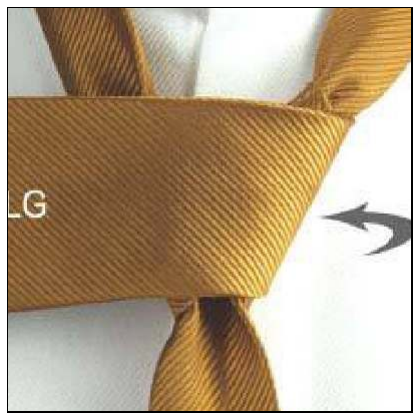}
    \caption{}
    \end{subfigure}
    \begin{subfigure}[b]{0.24\textwidth}
    \includegraphics[width=\textwidth]{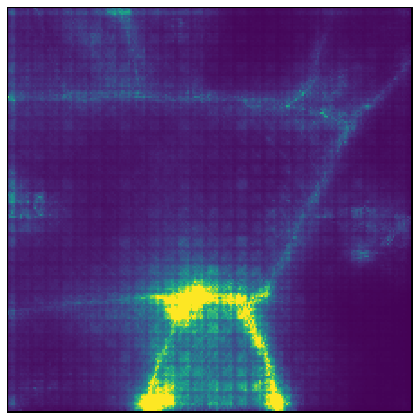}
    \caption{}
    \end{subfigure}
    \begin{subfigure}[b]{0.24\textwidth}
    \includegraphics[width=\textwidth]{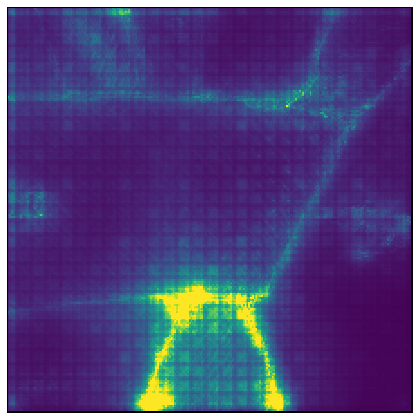}
    \caption{}
    \end{subfigure}
    \begin{subfigure}[b]{0.24\textwidth}
    \includegraphics[width=\textwidth]{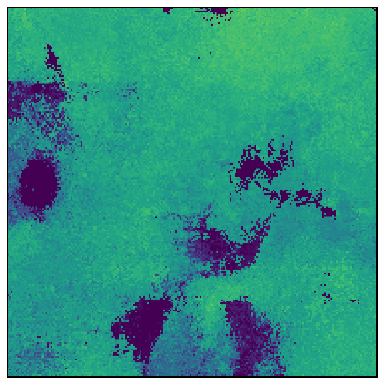}
    \caption{}
    \end{subfigure}
    \caption{The figure presents additional examples, including SmoothGrad$^2$ at maximum cosine similarity (b), SpectralLens$^2$ (c), and ArgLens (d) for ResNet50 trained on ImageNet. While SmoothGrad$^2$ may exhibit inconsistencies, as demonstrated in \cref{fig:example-sgsq-row-appendix}, leveraging cosine similarity provides a consistent means to determine its hyperparameters. Consequently, SmoothGrad$^2$ demonstrates consistency and upholds the visual quality associated with SmoothGrad$^2$ (SG$^2$).
    Addressing the inconsistency issue, another approach involves employing an explanation prior and aggregating information across different frequency bands, which we denote as SpectralLens (SL). Moreover, while SG$^2$ or SL effectively visualize the actual contributions, ArgLens enables the visualization of the contributions' sensitivity to noise. For AL we have visualized the frequencies obtained by $\omega = 1/ (1+\mathrm{AL})$.}
    \label{fig:example-al-row}
\end{figure*}

\begin{figure*}[h]
    \centering
    \includegraphics[width=0.24\textwidth]{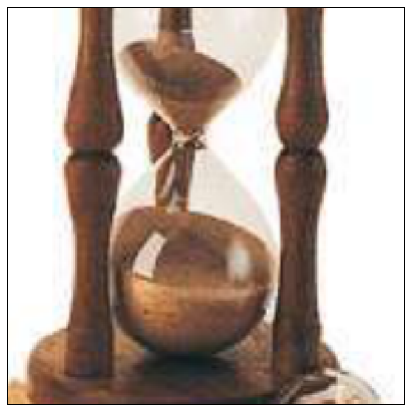}
    \includegraphics[width=0.24\textwidth]{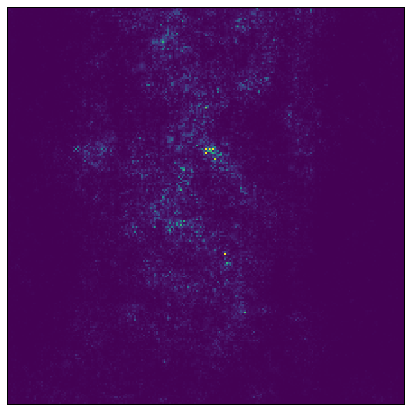}
    \includegraphics[width=0.24\textwidth]{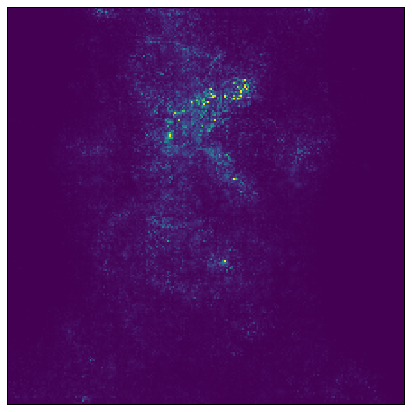}
    \includegraphics[width=0.24\textwidth]{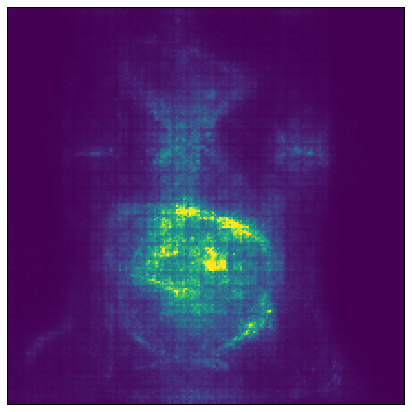}
    \includegraphics[width=0.24\textwidth]{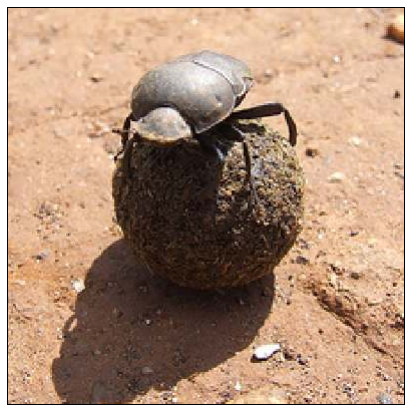}
    \includegraphics[width=0.24\textwidth]{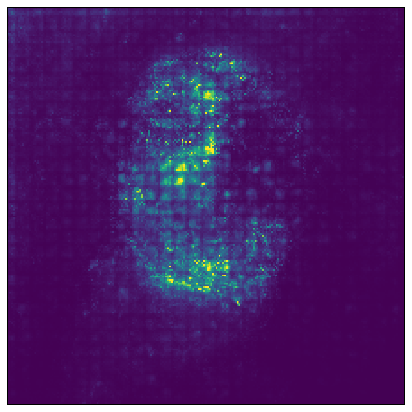}
    \includegraphics[width=0.24\textwidth]{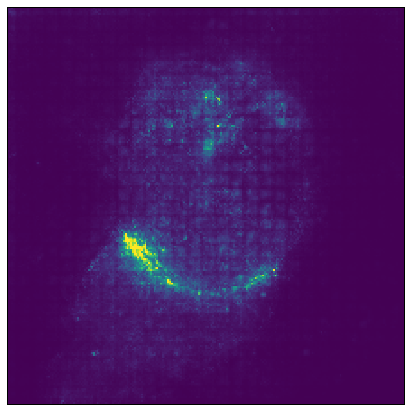}
    \includegraphics[width=0.24\textwidth]{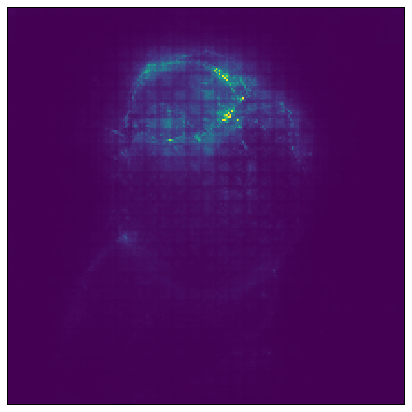}
    \includegraphics[width=0.24\textwidth]{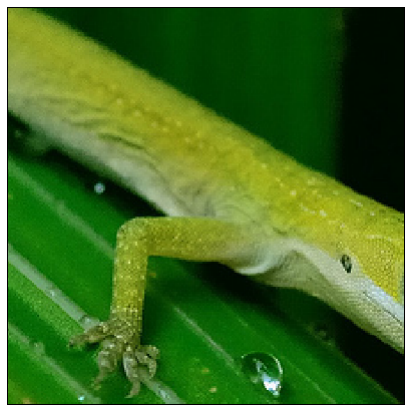}
    \includegraphics[width=0.24\textwidth]{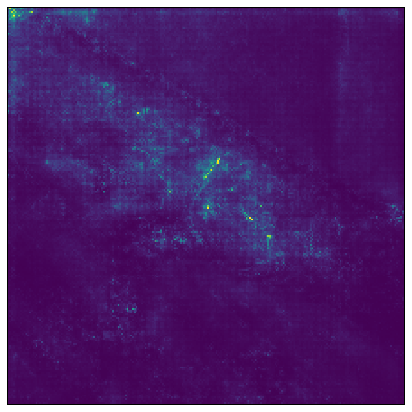}
    \includegraphics[width=0.24\textwidth]{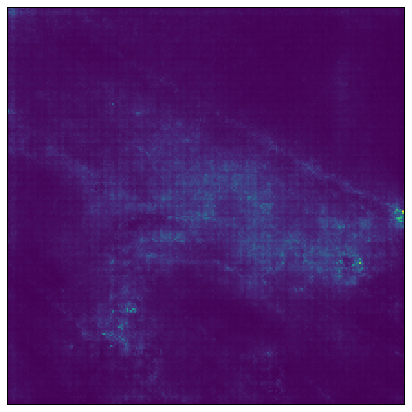}
    \includegraphics[width=0.24\textwidth]{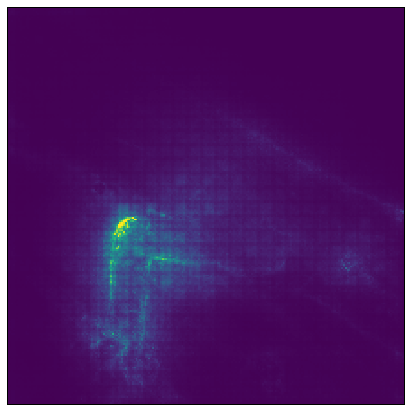}
    \includegraphics[width=0.24\textwidth]{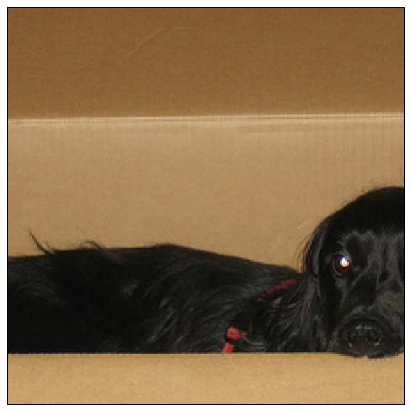}
    \includegraphics[width=0.24\textwidth]{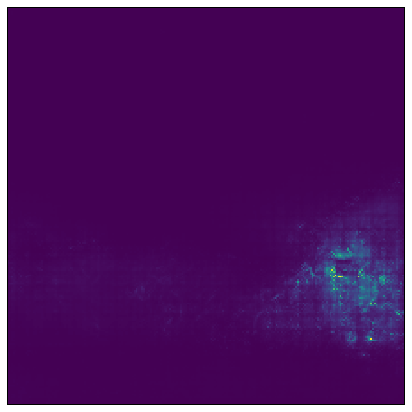}
    \includegraphics[width=0.24\textwidth]{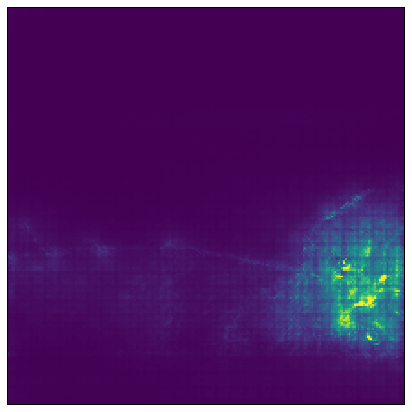}
    \includegraphics[width=0.24\textwidth]{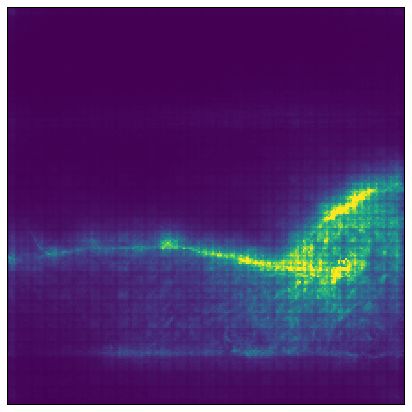}
    \begin{subfigure}[b]{0.24\textwidth}
    \includegraphics[width=\textwidth]{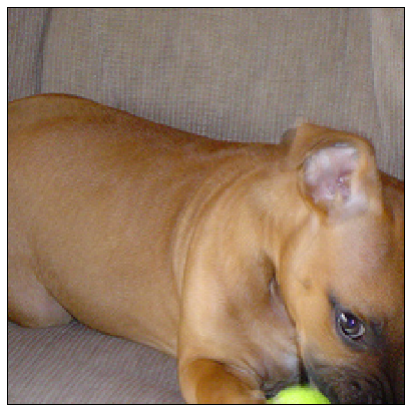}
    \caption{}
    \end{subfigure}
    \begin{subfigure}[b]{0.24\textwidth}
    \includegraphics[width=\textwidth]{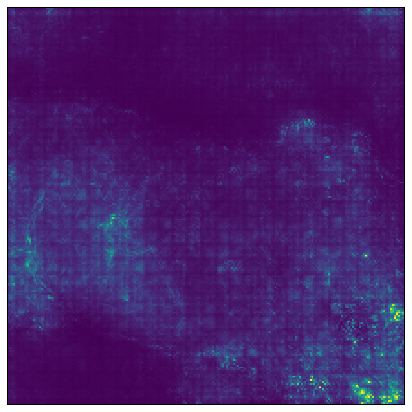}
    \caption{}
    \end{subfigure}
    \begin{subfigure}[b]{0.24\textwidth}
    \includegraphics[width=\textwidth]{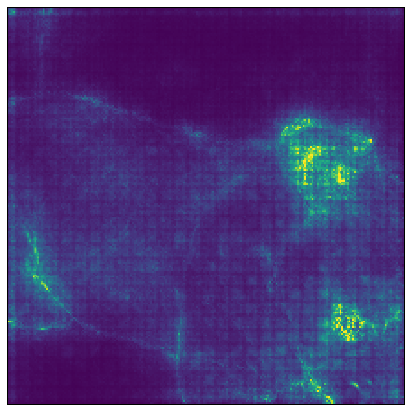}
    \caption{}
    \end{subfigure}
    \begin{subfigure}[b]{0.24\textwidth}
    \includegraphics[width=\textwidth]{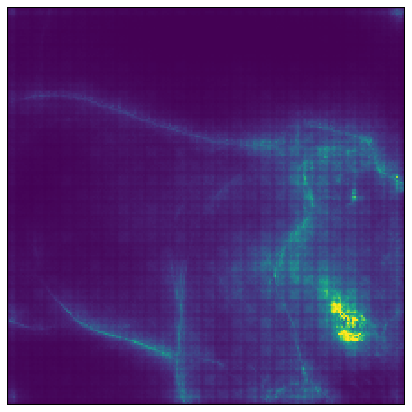}
    \caption{}
    \end{subfigure}
    \caption{The figure illustrates sample explanations generated by varying hyperparameters of SmoothGrad-Squared \cite{hooker_benchmark_2019,smilkov_smoothgrad_2017} for ResNet-50 on ImageNet. It is evident from the figure that using different hyperparameters results in diverse explanations. Such variations have been demonstrated in recent studies for other explanation techniques as well \cite{kindermans_reliability_2017,alvarez-melis_robustness_2018,brughmans_disagreement_2023,rong_consistent_2022}.
    }
    \label{fig:example-sgsq-row-appendix}
\end{figure*}
\end{document}